\documentclass[11pt]{article}

\usepackage[margin=1in]{geometry}
\usepackage[T1]{fontenc}
\usepackage[utf8]{inputenc}

\usepackage{mathtools,amssymb,amsthm}
\usepackage{enumitem}
\usepackage{microtype}
\usepackage{natbib}
\usepackage{algorithm}
\usepackage{algpseudocode}

\usepackage[
  colorlinks=true,
  linkcolor=blue,
  citecolor=blue,
  urlcolor=blue
]{hyperref}

\allowdisplaybreaks[3]

\theoremstyle{plain}
\newtheorem{theorem}{Theorem}[section]
\newtheorem{lemma}[theorem]{Lemma}
\newtheorem{proposition}[theorem]{Proposition}
\newtheorem{corollary}[theorem]{Corollary}

\theoremstyle{remark}
\newtheorem{remark}[theorem]{Remark}

\numberwithin{equation}{section}

\setlist[itemize]{
  leftmargin=*,
  topsep=4pt,
  itemsep=2pt
}
\setlist[enumerate]{
  leftmargin=*,
  topsep=4pt,
  itemsep=2pt
}

\title{R\'enyi Tracking Bounds for Langevin Dynamics with Moving Targets}
\author{Yuchen Xin\textsuperscript{*},
Jingxin Zhan\textsuperscript{$\dagger$},
Zhihua Zhang\textsuperscript{$\ddagger$}}
\date{}

\begin{document}

\maketitle

\begingroup
\renewcommand{\thefootnote}{\fnsymbol{footnote}}
\footnotetext[1]{School of Mathematical Sciences, Peking University; email: \texttt{2301110087@pku.edu.cn}}
\footnotetext[2]{School of Mathematical Sciences, Peking University; email: \texttt{bjdxzjx@pku.edu.cn}}
\footnotetext[3]{School of Mathematical Sciences, Peking University; email: \texttt{zhzhang@math.pku.edu.cn}}
\endgroup

\begin{abstract}
We study Langevin diffusion and Langevin Monte Carlo (LMC) when the target distribution changes over time. Under a log-Sobolev inequality (LSI), we derive non-asymptotic R\'enyi-divergence guarantees for tracking the current target. The framework covers continuous-time Langevin diffusion and its discretizations. We then apply the results to nonsmooth sampling based on successive Moreau envelopes. For this scheme, we give explicit choices of the smoothing parameters and step sizes, together with corresponding complexity bounds. To our knowledge, these are the first non-asymptotic R\'enyi-divergence tracking bounds for Langevin dynamics with discrete target updates.
\end{abstract}

\tableofcontents

\section{Introduction}

We study the problem of sampling from a target distribution $\pi\propto {\rm exp}(-V)$ on $\mathbb{R}^d$. Sampling from such distributions is a fundamental task in many areas such as statistics, machine learning, and computational physics \citep{robert2004monte, gelman1995bayesian}. In many applications, the normalizing constant of $\pi$ is intractable, making direct sampling impossible. 

To address these difficulties, a prominent line of work focuses on constructing Markov chain Monte Carlo (MCMC) algorithms that use gradient information of the potential $V$ \citep{brooks2011handbook}.  Among these, Langevin-based methods have emerged as a powerful approach, as they combine local gradient-driven exploration with stochastic perturbations to efficiently approximate the target distribution. The Langevin diffusion is defined by the stochastic differential equation
\[
dX_t = -\nabla V(X_t)\,dt + \sqrt{2}\,dB_t,
\]
where $(B_t)_{t \ge 0}$ is a standard Brownian motion. Under suitable conditions on the potential $V$, this diffusion admits $\pi$ as its invariant distribution, and thus provides a continuous-time mechanism for sampling from $\pi$. In practice, one cannot simulate this diffusion exactly and instead uses a time discretization, leading to the Langevin Monte Carlo (LMC) algorithm:
\[
X_{k+1} = X_k - h \nabla V(X_k) + \sqrt{2h}\,\xi_k,
\]
where $(\xi_k)_k$ are i.i.d.\ standard Gaussian random variables and $h>0$ is the step size.

Whilst classical LMC targets a fixed distribution $\pi$, a growing body of recent work considers sampling schemes with moving targets. Such approaches naturally arise in several contexts, including simulated annealing \citep{song2019generative}, and Moreau–Yosida regularization \citep{durmus2016efficientbayesiancomputationproximal, habring2025diffusionabsolutezerolangevin}, where one introduces a family of auxiliary distributions $(\pi_t)_{t\ge 0}$ that gradually approach the target $\pi$. 

Motivated by these developments, one can consider a time-inhomogeneous Langevin diffusion
\begin{align}
dX_t = -\nabla V_t(X_t)\,dt + \sqrt{2}\,dB_t,
\label{eq:smooth moving target diffusion}
\end{align}
whose associated Gibbs measure at time $t$ is given by $\pi_t \propto \exp(-V_t)$.

In this work, we first study such time-inhomogeneous Langevin dynamics in continuous time, where the potential may evolve smoothly with time. This allows us to identify the basic mechanism behind the analysis: the Rényi divergence contracts under the log-Sobolev inequality, while the time variation of the target distribution creates an additional error term. We then focus on a tractable and practically relevant subclass of such dynamics, where the potential evolves in a piecewise constant manner. Specifically, let $0 = t_0 < t_1 < \dots < t_M = T$ be a partition of the time interval, and let $\lambda_1,\dots,\lambda_M$ be given parameters. We consider the Langevin diffusion with moving targets
\begin{align}
dX_t = -\nabla V_{\lambda_i}(X_t)\,dt + \sqrt{2}\,dB_t,\qquad t\in[t_{i-1},t_i).
\label{Langevin diffusion with moving targets}
\end{align}

This formulation can be viewed as a time discretization of a general inhomogeneous Langevin diffusion, and naturally captures annealing-type or regularization-based schemes where the potential is updated at discrete stages.

We then consider the following discretization. For $1\le i\le M$, we apply the unadjusted Langevin algorithm (ULA) with step size $h_i = \frac{t_i - t_{i-1}}{N_i}$:
\begin{equation}
\begin{aligned}
X_{t_{i-1}+kh_i} 
&= X_{t_{i-1}+(k-1)h_i}
- h_i\nabla V_{\lambda_i}(X_{t_{i-1}+(k-1)h_i}) \\
&\quad + \sqrt{2}\,(B_{t_{i-1}+kh_i}-B_{t_{i-1}+(k-1)h_i}),
\quad k=1,\cdots,N_i.
\end{aligned}
\label{ULA with moving targets}
\end{equation}

Despite the growing interest in time-inhomogeneous Langevin dynamics, their theoretical understanding remains limited and fragmented. Many techniques from the homogeneous setting do not directly extend to this time-inhomogeneous setting. Existing analyses are often tailored to specific paths and rely on problem-dependent arguments. Moreover, most guarantees are based on the Kullback–Leibler divergence. We defer a more detailed discussion to related work.

In this work, we study this problem using R\'enyi divergence. We first analyze the general time-inhomogeneous Langevin diffusion \eqref{eq:smooth moving target diffusion}, and then develop iterative bounds for the piecewise-constant diffusion in \eqref{Langevin diffusion with moving targets} and its LMC discretization in \eqref{ULA with moving targets}.

\paragraph{Contributions.} Our contributions are summarized as follows.
\begin{itemize}
    \item We analyze the convergence of Langevin diffusion with moving targets and its discretization under a log-Sobolev inequality (LSI). For the general time-inhomogeneous diffusion \eqref{eq:smooth moving target diffusion}, Theorem~\ref{thm:smooth-moving-targets} gives a differential inequality for the R\'enyi divergence between the law of the process and the instantaneous target distribution. For the piecewise-constant diffusion in \eqref{Langevin diffusion with moving targets} and its LMC discretization in \eqref{ULA with moving targets}, Theorem~\ref{thm:Renyi convergence for ULA with moving targets} and Theorem~\ref{thm:Renyi convergence for Langevin diffusion with moving targets} give iterative non-asymptotic upper bounds.
    
    \item We analyze LMC using successive Moreau envelopes. We propose a sampling scheme in \autoref{alg:ULA using successive Moreau Envelopes} and establish a non-asymptotic complexity bound for the algorithm in Theorem~\ref{thm:Convergence of ULA using successive Moreau Envelopes}.
\end{itemize}

\paragraph{Related work}

The convergence of Langevin diffusion and its discretizations, such as Langevin Monte Carlo (LMC), has been extensively studied under various assumptions on the potential (e.g., \citet{roberts1996exponential}; \citet{durmus2016nonasymptoticconvergenceanalysisunadjusted}; \citet{dalalyan2016theoreticalguaranteesapproximatesampling}). These works primarily focus on convergence toward a fixed target distribution, and are typically analyzed in Kullback–Leibler (KL) divergence, total variation, or Wasserstein distance. More recently, convergence guarantees in R\'enyi divergence have also been established (e.g., \citet{vempala2022rapidconvergenceunadjustedlangevin}; \citet{chewi2024analysislangevinmontecarlo}).

Recently, the theoretical analysis of time-inhomogeneous Langevin diffusion and LMC with moving targets has received attention. Many existing works are tailored to specific problem settings or annealing paths. For instance, \citet{corderoencinar2025nonasymptoticanalysisdiffusionannealed} analyzed diffusion path, closely related to generative modelling techniques
based on annealed Langevin dynamic in \citep{song2019generative}, and established convergence by bounding the backward KL divergence via Wasserstein metric derivatives along the path of distributions. Similarly, \citet{guo2025provablebenefitannealedlangevin} considered a class of interpolating distributions where both the energy function and an additional quadratic regularization term vary along the path, and analyzed convergence by controlling the backward KL divergence using  the Girsanov theorem and optimal transport techniques. In a related direction, \citet{chehab2025provableconvergencelimitationsgeometric} studied geometric tempering paths and analyzed the evolution of the forward KL divergence along the dynamics. More recently, \citet{habring2026forwardklconvergencetimeinhomogeneouslangevin} established non-asymptotic forward-KL convergence for continuously varying Langevin diffusions and their Euler--Maruyama discretizations under uniform regularity, dissipativity, and log-Sobolev assumptions. They also derived worst-case complexity bounds and qualitative guidelines for path and step-size design. Our results are complementary. We study tracking in R\'enyi divergence and, in the piecewise-constant setting, allow the smoothness and log-Sobolev constants to vary from stage to stage, while measuring each target switch through a R\'enyi divergence between consecutive targets.

As an important instance of Langevin Monte Carlo with moving targets, we consider sampling based on successive Moreau envelopes, which has recently attracted increasing attention for handling nonsmooth potentials. A common approach in this setting is to replace the original potential with its Moreau envelope, yielding a differentiable surrogate that enables gradient-based Langevin dynamics. This idea underlies the Moreau–Yosida unadjusted Langevin algorithm (MYULA), where a fixed Moreau parameter is used to define a single surrogate distribution \citep{durmus2016efficientbayesiancomputationproximal,pereyra2015proximalmarkovchainmonte}. While this approach improves regularity, it introduces a non-vanishing bias unless the parameter is taken to zero.

This limitation naturally motivates considering a sequence of target distributions corresponding to varying Moreau parameters, which can be viewed as a particular case of Langevin Monte Carlo with moving targets. Such an approach has been explored in \citet{habring2025diffusionabsolutezerolangevin}, where the authors studied a sequence of Moreau envelopes that progressively approximate the target distribution. They established ergodicity of the associated Langevin dynamics for each fixed parameter and gave stability properties of the corresponding stationary distributions, including Lipschitz continuity in total variation with respect to the Moreau parameter. However, they did not provide a non-asymptotic error bound for the complete algorithm when the Moreau parameter changes over time. In their later work, \citet{habring2026forwardklconvergencetimeinhomogeneouslangevin} considered Langevin dynamics with time-varying Moreau parameters as an example. Nevertheless, their Moreau analysis assumes a differentiable potential and does not provide explicit stagewise choices of the Moreau parameters and step sizes for the discrete scheme considered here. Our framework provides a way to analyze Langevin dynamics with time-varying Moreau parameters, and enables the derivation of explicit convergence guarantees.

\section{Preliminaries}

\paragraph{R\'enyi divergence.}
Let $q \in (1,\infty)$ and let $\mu,\nu$ be probability measures.
The R\'enyi divergence of order $q$ is defined as
\[
\mathcal{R}_q(\mu \| \nu)
:= \frac{1}{q-1} \log \int \left(\frac{d\mu}{d\nu}\right)^q d\nu,
\]
if $\mu \ll \nu$ and $\mathcal{R}_q(\mu\|\nu)=+\infty$ otherwise.
The cases $q=1$ and $q=\infty$ are defined by limits:
$\mathcal{R}_1(\mu\|\nu)=\mathrm{KL}(\mu\|\nu)$ and
$\mathcal{R}_\infty(\mu\|\nu)=\log\|\frac{d\mu}{d\nu}\|_{L^\infty(\nu)}$.
Although $\mathcal{R}_q$ is not a metric, it satisfies $\mathcal{R}_q(\mu\|\nu)\ge 0$ with equality if and only if $\mu=\nu$.

We recall the following standard properties of R\'enyi divergence (see, e.g., \citet[Proposition~A.2]{altschuler2024shiftedcompositioniiilocal}).

\begin{lemma}
\label{lemma:renyi-basic}
Let $q \in [1,\infty]$ and let $\mu,\nu,\pi$ be probability measures.
\begin{enumerate}
\item \textbf{(Monotonicity)}
If $1 \le q \le q'$, then
\[
\mathcal{R}_q(\mu\|\nu) \le \mathcal{R}_{q'}(\mu\|\nu).
\]

\item \textbf{(Weak triangle inequality)}
For any $\lambda \in (0,1)$,
\[
\mathcal{R}_q(\mu \| \pi)
\le
\frac{q-\lambda}{q-1} \mathcal{R}_{\frac{q}{\lambda}}(\mu \| \nu)
+
\mathcal{R}_{\frac{q-\lambda}{1-\lambda}}(\nu \| \pi).
\]
\end{enumerate}
\end{lemma}

\paragraph{Log-Sobolev inequality.}

A probability measure $\pi$ on $\mathbb{R}^d$ is said to satisfy a log-Sobolev inequality (LSI) with constant $C_{\mathrm{LSI}} > 0$ if for all smooth functions $f : \mathbb{R}^d \to \mathbb{R}$ with $\mathbb{E}_\pi[f^2] > 0$,
\[
\operatorname{Ent}_\pi(f^2)
:= \mathbb{E}_\pi\!\left[f^2 \log \frac{f^2}{\mathbb{E}_\pi[f^2]}\right]
\le 2 C_{\mathrm{LSI}} \, \mathbb{E}_\pi \|\nabla f\|^2.
\]
The smallest such constant $C_{\mathrm{LSI}}$ is called the log-Sobolev constant of $\pi$.

\paragraph{Moreau envelope and proximal mapping.}\label{para:Moreau}
Let $g:\mathbb{R}^d \to \mathbb{R}$ be a convex function and let $\lambda>0$. 
The Moreau envelope of $g$ is defined by
\[
g_\lambda(x)
:=
\inf_{y\in\mathbb{R}^d}
\left\{
g(y)+\frac{1}{2\lambda}\|x-y\|^2
\right\},
\]
and the associated proximal mapping is defined by
\[
\operatorname{prox}_{\lambda g}(x)
:=
\arg\min_{y\in\mathbb{R}^d}
\left\{
g(y)+\frac{1}{2\lambda}\|x-y\|^2
\right\}.
\]

We recall several standard properties (see, e.g., \citet{rockafellar1998variational, durmus2016efficientbayesiancomputationproximal}).

For convex $g$ and $\lambda>0$, the proximal mapping is nonexpansive:
\[
\|\operatorname{prox}_{\lambda g}(x)-\operatorname{prox}_{\lambda g}(y)\|
\le
\|x-y\|.
\]
And $g_\lambda$ inherits the convexity of $g$ and is differentiable with
\[
\nabla g_\lambda(x)
=
\frac{1}{\lambda}\bigl(x-\operatorname{prox}_{\lambda g}(x)\bigr).
\]
Moreover, $g_\lambda$ is $1/\lambda$-smooth:
\[
\|\nabla g_\lambda(x)-\nabla g_\lambda(y)\|
\le
\frac{1}{\lambda}\|x-y\|.
\]
The derivative of the Moreau envelope with respect to $\lambda$ is given by (see, e.g., \citet[Lemma~4.18]{habring2025diffusionabsolutezerolangevin})
\[
\frac{\partial g_\lambda}{\partial\lambda}(x)
=
-\frac{1}{2\lambda^2}
\|x-\operatorname{prox}_{\lambda g}(x)\|^2.
\]

\section{Main Results}

In this section, we provide a framework to analyze the error in the R\'enyi divergence for Langevin dynamics with moving targets. We first discuss the continuous-time diffusion with a smoothly evolving target distribution. This reveals the basic mechanism behind our analysis: the R\'enyi divergence contracts because of the log-Sobolev inequality, while the time variation of the target distribution creates an additional drift term. We then turn to the discretized setting in \eqref{ULA with moving targets}, where the targets are updated in a piecewise constant manner.

We begin with the following differential inequality for the time-inhomogeneous Langevin diffusion \eqref{eq:smooth moving target diffusion}.
For each $t\ge 0$, let
\[
\pi_t(dx)=\frac{e^{-V_t(x)}}{Z_t}\,dx,
\qquad
\rho_t=\frac{d\mu_t}{d\pi_t},
\]
where $\mu_t$ denotes the law of $X_t$. The following result controls the instantaneous change of
$\mathcal{R}_q(\mu_t\|\pi_t)$.

\begin{theorem}
\label{thm:smooth-moving-targets}
Consider the time-inhomogeneous Langevin diffusion in \eqref{eq:smooth moving target diffusion}.
Fix $q\in(1,\infty)$. Assume that, for every $t\ge 0$, the target distribution $\pi_t$ satisfies a log-Sobolev inequality with constant $C_{\mathrm{LSI}}(\pi_t)$. Assume moreover that there exists $\alpha_t>\frac{q\,C_{\mathrm{LSI}}(\pi_t)}{2}$ such that
\[
\mathbb{E}_{\pi_t}\exp\!\left(\alpha_t|\partial_t V_t|\right)<\infty .
\]
Assume further that $(t,x)\mapsto V_t(x)$ and the density $p_t(x)$ of $\mu_t$ are $C^{1,2}$, that the relevant integrals are finite and may be differentiated under the integral sign, and that the boundary terms in the integrations by parts vanish at infinity. 

Define the centered time derivative of the potential by
\[
f_t(x)
=
\partial_t V_t(x)-\mathbb{E}_{\pi_t}[\partial_t V_t].
\]
Then, for any measurable choice of $\beta_t$ satisfying
\[
\frac{q\,C_{\mathrm{LSI}}(\pi_t)}{2}<\beta_t<\alpha_t,
\]
we have
\begin{align}
\frac{d}{dt}\mathcal{R}_q(\mu_t\|\pi_t)
\le
-
\left(
\frac{2}{q\,C_{\mathrm{LSI}}(\pi_t)}
-\frac{1}{\beta_t}
\right)
\mathcal{R}_q(\mu_t\|\pi_t)
+
\frac{1}{\beta_t}
\log
\mathbb{E}_{\pi_t}
\exp\!\left(\beta_t f_t\right).
\label{eq:renyi-dissipation-moving-target}
\end{align}
\end{theorem}

\begin{proof}
Define
\[
Z_q(t)=\int \rho_t^q\,d\pi_t.
\]
Then
\[
\mathcal{R}_q(\mu_t\|\pi_t)
=
\frac{1}{q-1}\log Z_q(t).
\]
The key point is to keep track of two effects separately. The first one comes from the Langevin evolution with the target $\pi_t$ frozen at time $t$, and the second one comes from the motion of the target distribution itself.

Let $p_t$ denote the density of $\mu_t$ with respect to the Lebesgue measure. Then $p_t=\rho_t\pi_t$. The Fokker--Planck equation for the diffusion gives
\[
\partial_t p_t
=
\nabla\cdot(p_t\nabla V_t)+\Delta p_t.
\]
Using $p_t=\rho_t\pi_t$ and $\pi_t\propto e^{-V_t}$, the spatial part can be rewritten as
\[
\nabla\cdot(p_t\nabla V_t)+\Delta p_t
=
\pi_t\bigl(\Delta\rho_t-\nabla V_t\cdot\nabla\rho_t\bigr).
\]
On the other hand,
\[
\partial_t p_t
=
(\partial_t\rho_t)\pi_t+\rho_t\,\partial_t\pi_t.
\]
Since
\[
\partial_t\log\pi_t
=
-\partial_t V_t+\mathbb{E}_{\pi_t}[\partial_t V_t]
=
-f_t,
\]
we obtain
\[
\partial_t\rho_t
=
\Delta\rho_t-\nabla V_t\cdot\nabla\rho_t
+
\rho_t f_t.
\]
This identity is the only place where the time dependence of the target enters the proof.

Differentiating $Z_q(t)$ yields
\begin{align*}
\frac{d}{dt}Z_q(t)
&=
q\int \rho_t^{q-1}\partial_t\rho_t\,d\pi_t
+
\int \rho_t^q\,\partial_t\log\pi_t\,d\pi_t  \\
&=
q\int \rho_t^{q-1}
\bigl(\Delta\rho_t-\nabla V_t\cdot\nabla\rho_t\bigr)\,d\pi_t
+
(q-1)\int \rho_t^q f_t\,d\pi_t .
\end{align*}
By integration by parts with respect to $\pi_t$,
\[
\int \rho_t^{q-1}
\bigl(\Delta\rho_t-\nabla V_t\cdot\nabla\rho_t\bigr)\,d\pi_t
=
-(q-1)
\int \rho_t^{q-2}\|\nabla\rho_t\|^2\,d\pi_t .
\]
Therefore
\[
\frac{d}{dt}Z_q(t)
=
-q(q-1)
\int \rho_t^{q-2}\|\nabla\rho_t\|^2\,d\pi_t
+
(q-1)
\int \rho_t^q f_t\,d\pi_t .
\]
Dividing by $(q-1)Z_q(t)$ gives
\begin{align}
\frac{d}{dt}\mathcal{R}_q(\mu_t\|\pi_t)
=
-q
\frac{
\int \rho_t^{q-2}\|\nabla\rho_t\|^2\,d\pi_t
}{
\int \rho_t^q\,d\pi_t
}
+
F_t,
\label{eq:renyi-derivative-before-dv}
\end{align}
where
\[
F_t
=
\frac{
\int \rho_t^q f_t\,d\pi_t
}{
\int \rho_t^q\,d\pi_t
}.
\]
The first term in \eqref{eq:renyi-derivative-before-dv} is the usual dissipative term. The second term is new and measures the instantaneous change of the target distribution.

To bound $F_t$, introduce the escort probability measure
\[
d\nu_{q,t}
=
\frac{\rho_t^q}
{\int \rho_t^q\,d\pi_t}
\,d\pi_t.
\]
Then
\[
F_t=\mathbb{E}_{\nu_{q,t}}[f_t].
\]
By the Donsker--Varadhan variational inequality, for any $\beta_t<\alpha_t$,
\[
\mathbb{E}_{\nu_{q,t}}[f_t]
\le
\frac{1}{\beta_t}
\mathrm{KL}(\nu_{q,t}\|\pi_t)
+
\frac{1}{\beta_t}
\log
\mathbb{E}_{\pi_t}\exp(\beta_t f_t).
\]
It remains to control the KL term by the dissipative term in \eqref{eq:renyi-derivative-before-dv}. Define
\[
h_t
=
\sqrt{\frac{d\nu_{q,t}}{d\pi_t}}
=
\frac{\rho_t^{q/2}}
{
\left(\int \rho_t^q\,d\pi_t\right)^{1/2}
}.
\]
Then
\[
\mathrm{KL}(\nu_{q,t}\|\pi_t)
=
\operatorname{Ent}_{\pi_t}(h_t^2).
\]
Applying the log-Sobolev inequality for $\pi_t$ gives
\[
\mathrm{KL}(\nu_{q,t}\|\pi_t)
\le
2C_{\mathrm{LSI}}(\pi_t)
\int \|\nabla h_t\|^2\,d\pi_t .
\]
A direct computation gives
\[
\int \|\nabla h_t\|^2\,d\pi_t
=
\frac{q^2}{4}
\frac{
\int \rho_t^{q-2}\|\nabla\rho_t\|^2\,d\pi_t
}{
\int \rho_t^q\,d\pi_t
}.
\]
Hence
\[
\mathrm{KL}(\nu_{q,t}\|\pi_t)
\le
\frac{C_{\mathrm{LSI}}(\pi_t)q^2}{2}
\frac{
\int \rho_t^{q-2}\|\nabla\rho_t\|^2\,d\pi_t
}{
\int \rho_t^q\,d\pi_t
}.
\]
Substituting this estimate into the bound for $F_t$ yields
\[
F_t
\le
\frac{C_{\mathrm{LSI}}(\pi_t)q^2}{2\beta_t}
\frac{
\int \rho_t^{q-2}\|\nabla\rho_t\|^2\,d\pi_t
}{
\int \rho_t^q\,d\pi_t
}
+
\frac{1}{\beta_t}
\log
\mathbb{E}_{\pi_t}\exp(\beta_t f_t).
\]
Combining this inequality with \eqref{eq:renyi-derivative-before-dv}, we obtain
\begin{align}
\frac{d}{dt}\mathcal{R}_q(\mu_t\|\pi_t)
\le
-
\left(
q-\frac{C_{\mathrm{LSI}}(\pi_t)q^2}{2\beta_t}
\right)
\frac{
\int \rho_t^{q-2}\|\nabla\rho_t\|^2\,d\pi_t
}{
\int \rho_t^q\,d\pi_t
}
+
\frac{1}{\beta_t}
\log
\mathbb{E}_{\pi_t}\exp(\beta_t f_t).
\label{eq:renyi-derivative-gradient-bound}
\end{align}
Finally, by the standard R\'enyi dissipation estimate under LSI
\citep[Lemma~5]{vempala2022rapidconvergenceunadjustedlangevin},
\[
\frac{
\int \rho_t^{q-2}\|\nabla\rho_t\|^2\,d\pi_t
}{
\int \rho_t^q\,d\pi_t
}
\ge
\frac{2}{q^2C_{\mathrm{LSI}}(\pi_t)}
\mathcal{R}_q(\mu_t\|\pi_t).
\]
Since $\beta_t>qC_{\mathrm{LSI}}(\pi_t)/2$, the coefficient
\[
q-\frac{C_{\mathrm{LSI}}(\pi_t)q^2}{2\beta_t}
\]
is positive. Substituting the above lower bound into
\eqref{eq:renyi-derivative-gradient-bound} gives
\[
\frac{d}{dt}\mathcal{R}_q(\mu_t\|\pi_t)
\le
-
\left(
\frac{2}{qC_{\mathrm{LSI}}(\pi_t)}
-\frac{1}{\beta_t}
\right)
\mathcal{R}_q(\mu_t\|\pi_t)
+
\frac{1}{\beta_t}
\log
\mathbb{E}_{\pi_t}\exp(\beta_t f_t).
\]
This completes the proof.
\end{proof}

Theorem~\ref{thm:smooth-moving-targets} shows how the classical R\'enyi dissipation under LSI is modified when the target distribution changes with time. The first term in \eqref{eq:renyi-dissipation-moving-target} is the contraction term. The second term depends on the exponential moment of the centered quantity $f_t=\partial_tV_t-\mathbb{E}_{\pi_t}[\partial_tV_t]$, and therefore measures the speed of the target evolution. In particular, if the target is fixed, then $f_t=0$ and the usual exponential decay of R\'enyi divergence is recovered.

We next turn to the discretized setting in \eqref{ULA with moving targets}. In the algorithmic setting considered below, the target distribution is updated only at the grid points
\[
0=t_0<t_1<\cdots<t_M=T.
\]
Thus the potential is frozen on each interval $[t_{i-1},t_i)$ and changes from $V_{\lambda_i}$ to $V_{\lambda_{i+1}}$ at the end of the interval. This piecewise constant structure is convenient for analyzing LMC, since on each interval the dynamics behaves like Langevin dynamics for a fixed target, while the change of target can be controlled through a R\'enyi divergence term between consecutive targets.

In order to analyze the R\'enyi divergence for \eqref{ULA with moving targets}, we want to track the change of $\mathcal{R}_{2}(\mu_{t_i}\|\pi_{\lambda_{i+1}})$ where $\mu_{t_i}$ is the law of $X_{t_i}$. In other words, we need to bound $\mathcal{R}_{2}(\mu_{t_i}\|\pi_{\lambda_{i+1}})$ in terms of $\mathcal{R}_{2}(\mu_{t_{i-1}}\|\pi_{\lambda_{i}})$. Our strategy is to use weak triangle inequality to bound $\mathcal{R}_{2}(\mu_{t_i}\|\pi_{\lambda_{i+1}})$ by $\mathcal{R}_{q_1}(\mu_{t_{i}}\|\pi_{\lambda_{i}})$ and $\mathcal{R}_{q_2}(\pi_{\lambda_{i}}\|\pi_{\lambda_{i+1}})$ for some $q_1, q_2$. We use order $2$ as the recursively propagated R\'enyi order for simplicity. The same argument extends to any fixed order $p\ge2$; see Remark~\ref{rem:general-renyi-orders} and Appendix~\ref{appendix:general-renyi-orders}.

To bound $\mathcal{R}_{q_1}(\mu_{t_{i}}\|\pi_{\lambda_{i}})$ in terms of $\mathcal{R}_{2}(\mu_{t_{i-1}}\|\pi_{\lambda_{i}})$, we need a R\'enyi divergence bound for LMC. Note that we need to choose $q_1>2$ in order to use weak triangle inequality. \citet[Theorem~4]{chewi2024analysislangevinmontecarlo} provided a result in such form. However, they only give the bound when the number of iterations exceeds a threshold $N_0$, making it unsuitable for direct application to our analysis of LMC with moving targets. Thus we give the following lemma by adapting their proof technique.

\begin{lemma}
\label{lem:Renyi convergence for ULA}
Let $\mu_t$ be the law of the continuous interpolation of LMC with the target distribution $\pi \propto e^{-V}$ on $\mathbb{R}^d$
    \[
    x_t = x_{kh} - (t-kh) \nabla V(x_{kh}) + \sqrt{2} (B_t-B_{kh}),\qquad t\in[kh,(k+1)h]
    \]
    and $\rho_t = d\mu_t/d\pi$. 
Assume that the target distribution $\pi \propto e^{-V}$ satisfies a log-Sobolev inequality (LSI) with constant $C_{\mathrm{LSI}}\ge 1$ and the gradient $\nabla V$ is $L$-Lipschitz with $L\ge 1$. The step size $h>0$ satisfies $h \le \frac{1}{10000\,C_{\mathrm{LSI}}L^2}$.

Then, for the first $N$ steps from $0$ to $Nh$ with $Nh\le 2C_{\mathrm{LSI}}\ln2$, the following recursion holds:
\begin{align*}
\mathcal{R}_{1 + e^{Nh/(2C_{\mathrm{LSI}})}}(\mu_{Nh}\|\pi) \le e^{-Nh/(6C_{\mathrm{LSI}})}\mathcal{R}_{2}(\mu_0\|\pi) + 200NdL^2h^2.
\end{align*}

\end{lemma}

\begin{proof}[Sketch of the proof]
The proof follows the proof of \citet[Theorem~4]{chewi2024analysislangevinmontecarlo}. 

Define the time-varying Rényi order $q(t) = 1 + e^{t/(2C_{\mathrm{LSI}})}$ and $\Phi(t) = \frac{1}{q(t)}\ln \int \rho_t^{q(t)}\,d\pi$, so that $\mathcal{R}_{q(t)}(\mu_t\|\pi) = \frac{q(t)}{q(t)-1}\Phi(t)$.
Then we have $q(0)=2$ and $q(t)\in[2,3]$ for $0\le t\le 2C_{\mathrm{LSI}}\ln 2$.
For $t\in[kh,(k+1)h]$ with $(k+1)h\le 2C_{\mathrm{LSI}}\ln2$, \citet[Proposition~21]{chewi2024analysislangevinmontecarlo} reads
\begin{align}
\partial_t\Phi(t) \le -\frac{2(q(t)-1)}{q(t)^2}\,I(t) + (q(t)-1)D(t),\label{eq:differential inequality}
\end{align}
where $I(t)=\frac{\mathbb{E}_\pi[\|\nabla(\rho_t^{q(t)/2})\|^2]}{\mathbb{E}_\pi(\rho_t^{q(t)})}$ and $D(t)=\mathbb{E}[\psi_t(x_t)\|\nabla V(x_t)-\nabla V(x_{kh})\|^2]$, $\psi_t=\rho_t^{q(t)-1}/\mathbb{E}_\pi(\rho_t^{q(t)})$.

For $D(t)$, we use the upper bound in the proof of \citet[Theorem~4]{chewi2024analysislangevinmontecarlo}:
\begin{align}
D(t) \le 36L^2(t-kh)^2 I(t) + 192hL^2C_{\mathrm{LSI}}I(t) + 18dL^3(t-kh)^2 + 84dL^2(t-kh).\label{eq:upper bound for D}
\end{align}
For completeness, we include a brief proof of this result (following \citet{chewi2024analysislangevinmontecarlo}) in Appendix \ref{subsec:upper bound for D}.

Plugging \eqref{eq:upper bound for D} into the differential inequality \eqref{eq:differential inequality} and merging the $I(t)$ terms give
\begin{align*}
\partial_t\Phi(t) \le &(q(t)-1)\Bigl(-\frac{2}{q(t)^2}+36L^2(t-kh)^2+192hL^2C_{\mathrm{LSI}}\Bigr)I(t)\\ &+ 18(q(t)-1)dL^3(t-kh)^2+84(q(t)-1)dL^2(t-kh).
\end{align*}
The step size condition ensures $36L^2h^2+192hL^2C_{\mathrm{LSI}}\le 1/9 \le 1/q(t)^2$, hence the coefficient of $I$ is at most $-(q-1)/q^2$. 
Applying the LSI lower bound $I(t)\ge \frac{1}{2C_{\mathrm{LSI}}}\cdot\frac{q(t)}{q(t)-1}\Phi(t)$ in \citet[Lemma~5]{vempala2022rapidconvergenceunadjustedlangevin} yields
\[
\partial_t\Phi(t) \le -\frac{1}{2q(t)C_{\mathrm{LSI}}}\Phi(t) + 18(q(t)-1)dL^3(t-kh)^2 + 84(q(t)-1)dL^2(t-kh).
\]
Since $q(t)\in[2,3]$, we have $\frac{1}{2q(t)C_{\mathrm{LSI}}}\ge\frac{1}{6C_{\mathrm{LSI}}}$ and $q(t)-1\le 2$, so
\[
\partial_t\Phi(t) \le -\frac{1}{6C_{\mathrm{LSI}}}\Phi(t) + 36dL^3(t-kh)^2 + 168dL^2(t-kh).
\]
Integrating this differential inequality from $kh$ to $(k+1)h$ (with integrating factor $e^{t/(6C_{\mathrm{LSI}})}$) gives
\begin{align*}
\Phi((k+1)h) &\le e^{-h/(6C_{\mathrm{LSI}})}\Phi(kh) + 12dL^3h^3 + 84dL^2h^2\\
&\le e^{-h/(6C_{\mathrm{LSI}})}\Phi(kh)  + 100dL^2h^2.
\end{align*}
Thus, for $Nh\le 2C_{\mathrm{LSI}}\ln2$, 
\[
\Phi(Nh) \le e^{-Nh/(6C_{\mathrm{LSI}})}\Phi(0) + 100NdL^2h^2.
\]
Thus,
\begin{align*}
\mathcal{R}_{1 + e^{Nh/(2C_{\mathrm{LSI}})}}(\mu_{Nh}\|\pi) &\le \frac{q(Nh)}{q(Nh)-1}(e^{-Nh/(6C_{\mathrm{LSI}})}\frac{1}{2}\mathcal{R}_{2}(\mu_0\|\pi) + 100NdL^2h^2)\\
&\le e^{-Nh/(6C_{\mathrm{LSI}})}\mathcal{R}_{2}(\mu_0\|\pi) + 200NdL^2h^2.
\end{align*}
\end{proof}

The bound decomposes into an exponentially decaying term and a discretization error accumulated along the trajectory. A notable aspect of the bound is that the R\'enyi parameter is $q = 1 + e^{Nh/(2C_{\mathrm{LSI}})}$, which increases with time. This reflects the hypercontractive behavior induced by the LSI: as the dynamics evolve, one gains control of higher-order R\'enyi divergences, with a growth rate proportional to \(Nh/C_{\mathrm{LSI}}\).

The proof is adapted from \citet[Theorem~4]{chewi2024analysislangevinmontecarlo}, and can be viewed as an explicit extraction of the short-time regime of their argument. In particular, while \citet{chewi2024analysislangevinmontecarlo} ultimately fix the R\'enyi order \(q\) and analyze the large-time regime \(N \ge N_0\), we retain the time-dependent choice of \(q\) and restrict to $Nh \le 2C_{\mathrm{LSI}}\ln 2$, making the growth of the R\'enyi parameter explicit.

As a consequence of the above lemma on the convergence of LMC, we obtain the following error recursion for \eqref{ULA with moving targets}.
\begin{theorem}
\label{thm:Renyi convergence for ULA with moving targets}
Consider LMC with moving targets in \eqref{ULA with moving targets}.
Assume that the target distribution $\pi_{\lambda_i} \propto e^{-V_{\lambda_i}}$ on $\mathbb{R}^d$ satisfies a log-Sobolev inequality (LSI) with constant $C_{\mathrm{LSI}}(\pi_{\lambda_i})\ge 1$ and the gradient $\nabla V_{\lambda_i}$ is $L_{\lambda_i}$-Lipschitz with $L_{\lambda_i}\ge 1$. The step size $h_i>0$ satisfies $h_i \le \frac{1}{10000\,C_{\mathrm{LSI}}(\pi_{\lambda_i})L_{\lambda_i}^2}$.
Then, for \eqref{ULA with moving targets} with $N_ih_i\le 2C_{\mathrm{LSI}}(\pi_{\lambda_i})\ln2$, the following recursion holds:
\begin{align*}
\mathcal{R}_{2}(\mu_{t_i}\|\pi_{\lambda_{i+1}})
\le e^{-N_ih_i/(12C_{\mathrm{LSI}}(\pi_{\lambda_i}))}\mathcal{R}_{2}(\mu_{t_{i-1}}\|\pi_{\lambda_{i}}) + 400N_idL_{\lambda_i}^2h_i^2+\mathcal{R}_{\frac{24C_{\mathrm{LSI}}(\pi_{\lambda_i})}{N_ih_i}}(\pi_{\lambda_{i}}\|\pi_{\lambda_{i+1}}),
\end{align*}
where $\mu_{t_i}$ is the law of $X_{t_i}$.

\end{theorem}

\begin{proof}[Proof]
Applying Lemma~\ref{lem:Renyi convergence for ULA}, we have
\begin{align*}
\mathcal{R}_{1 + e^{N_ih_i/(2C_{\mathrm{LSI}}(\pi_{\lambda_i}))}}(\mu_{t_i}\|\pi_{\lambda_{i}}) 
\le e^{-N_ih_i/(6C_{\mathrm{LSI}}(\pi_{\lambda_i}))}\mathcal{R}_{2}(\mu_{t_{i-1}}\|\pi_{\lambda_{i}}) + 200N_idL_{\lambda_i}^2h_i^2.
\end{align*}

Now we can bound the R\'enyi divergence between $\mu_{t_i}$ and $\pi_{\lambda_{i+1}}$ as
\begin{align*}
\mathcal{R}_{2}(\mu_{t_i}\|\pi_{\lambda_{i+1}}) 
\le &(1+\frac{N_ih_i}{12C_{\mathrm{LSI}}(\pi_{\lambda_i})})\mathcal{R}_{\frac{2}{1-N_ih_i/(12C_{\mathrm{LSI}}(\pi_{\lambda_i}))}}(\mu_{t_{i}}\|\pi_{\lambda_{i}}) + \mathcal{R}_{1+\frac{1}{N_ih_i/(12C_{\mathrm{LSI}}(\pi_{\lambda_i}))}}(\pi_{\lambda_{i}}\|\pi_{\lambda_{i+1}})\\
\le &(1+\frac{N_ih_i}{12C_{\mathrm{LSI}}(\pi_{\lambda_i})})\mathcal{R}_{1 + e^{N_ih_i/(2C_{\mathrm{LSI}}(\pi_{\lambda_i}))}}(\mu_{t_{i}}\|\pi_{\lambda_{i}}) + \mathcal{R}_{\frac{24C_{\mathrm{LSI}}(\pi_{\lambda_i})}{N_ih_i}}(\pi_{\lambda_{i}}\|\pi_{\lambda_{i+1}})\\
\le &e^{-N_ih_i/(12C_{\mathrm{LSI}}(\pi_{\lambda_i}))}\mathcal{R}_{2}(\mu_{t_{i-1}}\|\pi_{\lambda_{i}}) + 400N_idL_{\lambda_i}^2h_i^2+\mathcal{R}_{\frac{24C_{\mathrm{LSI}}(\pi_{\lambda_i})}{N_ih_i}}(\pi_{\lambda_{i}}\|\pi_{\lambda_{i+1}}).
\end{align*}
The first inequality follows from the weak triangle inequality in Lemma~\ref{lemma:renyi-basic}. In the second inequality, we used monotonicity in Lemma~\ref{lemma:renyi-basic} and $\frac{N_ih_i}{C_{\mathrm{LSI}}(\pi_{\lambda_i})}\le 2\ln2$.
\end{proof}

\begin{remark}[General R\'enyi orders]
\label{rem:general-renyi-orders}
The use of $\mathcal R_2$ in
Theorem~\ref{thm:Renyi convergence for ULA with moving targets}
is not intrinsic to the argument. Fix $p\ge2$ and write
\[
m_p=2p-1,
\qquad
T_i=N_i h_i.
\]
If
\[
T_i\le2C_i\log2,
\qquad
h_i
\le
\frac{1}{10000\,C_iL_i^2}
\left(
\frac{3}{m_p}
\right)^2,
\]
where
\[
C_i=C_{\mathrm{LSI}}(\pi_{\lambda_i}),
\qquad
L_i=L_{\lambda_i},
\]
then the same argument gives
\[
\begin{aligned}
\mathcal R_p(\mu_{t_i}\|\pi_{\lambda_{i+1}})
\le
\exp\left(
-\frac{T_i}{4C_i m_p}
\right)
\mathcal R_p(\mu_{t_{i-1}}\|\pi_{\lambda_i})
+
400(p-1)N_i dL_i^2h_i^2
+
\mathcal R_{8C_i m_p/T_i}
(\pi_{\lambda_i}\|\pi_{\lambda_{i+1}}).
\end{aligned}
\]
Thus, propagating a larger R\'enyi order is possible, at the cost of
a step size of order $p^{-2}$, an initial $\mathcal R_p$ bound, and
control of consecutive targets at R\'enyi order
$O(pC_i/T_i)$. Setting $p=2$ recovers
Theorem~\ref{thm:Renyi convergence for ULA with moving targets}.
A proof is given in
Appendix~\ref{appendix:general-renyi-orders}.
\end{remark}

The result for \eqref{Langevin diffusion with moving targets} naturally arises as the continuous-time limit of \eqref{ULA with moving targets} as the step size goes to zero.
\begin{theorem}
\label{thm:Renyi convergence for Langevin diffusion with moving targets}
Consider Langevin diffusion with moving targets in \eqref{Langevin diffusion with moving targets}.
Assume that the target distribution $\pi_{\lambda_i} \propto e^{-V_{\lambda_i}}$ on $\mathbb{R}^d$ satisfies a log-Sobolev inequality (LSI) with constant $C_{\mathrm{LSI}}(\pi_{\lambda_i})\ge 1$ and the gradient $\nabla V_{\lambda_i}$ is Lipschitz.
Then, for \eqref{Langevin diffusion with moving targets} with $t_i-t_{i-1}\le 2C_{\mathrm{LSI}}(\pi_{\lambda_i})\ln2$, the following recursion holds:
\begin{align*}
\mathcal{R}_{2}(\mu_{t_i}\|\pi_{\lambda_{i+1}})
\le e^{-(t_i-t_{i-1})/(12C_{\mathrm{LSI}}(\pi_{\lambda_i}))}\mathcal{R}_{2}(\mu_{t_{i-1}}\|\pi_{\lambda_{i}}) +\mathcal{R}_{\frac{24C_{\mathrm{LSI}}(\pi_{\lambda_i})}{t_i-t_{i-1}}}(\pi_{\lambda_{i}}\|\pi_{\lambda_{i+1}}),
\end{align*}
where $\mu_{t_i}$ is the law of $X_{t_i}$.

\end{theorem}

\begin{proof}
Fix $i$, and write $\Delta_i=t_i-t_{i-1}, C_i=C_{\mathrm{LSI}}(\pi_{\lambda_i})$, and let $L_i\ge1$ be a Lipschitz constant of $\nabla V_{\lambda_i}$.

For $N\ge1$, set $h_N=\Delta_i/N$, and let
$\mu_{t_i}^{(N)}$ denote the endpoint law of the Euler--Maruyama
scheme on $[t_{i-1},t_i]$, initialized from
$\mu_{t_{i-1}}$. Since the drift is globally Lipschitz, the standard
convergence of Euler--Maruyama gives
\[
\mu_{t_i}^{(N)}
\Rightarrow
\mu_{t_i}
\qquad\text{as }N\to\infty.
\]

For all sufficiently large $N$,
Theorem~\ref{thm:Renyi convergence for ULA with moving targets}
applies and yields
\[
\begin{aligned}
\mathcal R_2(
\mu_{t_i}^{(N)}\|\pi_{\lambda_{i+1}}
)
\le{}
e^{-\Delta_i/(12C_i)}
\mathcal R_2(
\mu_{t_{i-1}}\|\pi_{\lambda_i}
)
+
\frac{
400dL_i^2\Delta_i^2
}{N}
+
\mathcal R_{24C_i/\Delta_i}
(
\pi_{\lambda_i}\|\pi_{\lambda_{i+1}}
).
\end{aligned}
\]

For fixed $\nu$, the map
$\mu\mapsto\mathcal R_2(\mu\|\nu)$ is lower semicontinuous under weak
convergence. Indeed,
\[
e^{\mathcal R_2(\mu\|\nu)}
=
\sup_{\varphi\in C_b(\mathbb R^d)}
\left\{
2\int\varphi\,d\mu
-
\int\varphi^2\,d\nu
\right\}.
\]
Consequently,
\[
\begin{aligned}
\mathcal R_2(
\mu_{t_i}\|\pi_{\lambda_{i+1}}
)
\le
\liminf_{N\to\infty}
\mathcal R_2(
\mu_{t_i}^{(N)}\|\pi_{\lambda_{i+1}}
)
\le
e^{-\Delta_i/(12C_i)}
\mathcal R_2(
\mu_{t_{i-1}}\|\pi_{\lambda_i}
)
+
\mathcal R_{24C_i/\Delta_i}
(
\pi_{\lambda_i}\|\pi_{\lambda_{i+1}}
),
\end{aligned}
\]
which proves the claim.
\end{proof}

Theorems~\ref{thm:Renyi convergence for ULA with moving targets} and~\ref{thm:Renyi convergence for Langevin diffusion with moving targets}
provide a unified analysis of Langevin dynamics under a moving target framework, where the target distribution evolves along a sequence $\{\pi_{\lambda_i}\}$. The exponential factor reflects the contraction of the R\'enyi divergence under the log-Sobolev inequality (LSI), which ensures that the dynamics rapidly approaches the current target distribution $\pi_{\lambda_i}$. The term
$\mathcal{R}_{\alpha}(\pi_{\lambda_i}\|\pi_{\lambda_{i+1}})$
captures the effect of the change in the target distribution, and quantifies the difficulty of tracking a sequence of distributions rather than a single stationary one.
When consecutive targets are close in R\'enyi divergence, this term remains small, allowing the dynamics to effectively follow the evolving target. In the discrete-time setting, an additional error of order $N_i d L_{\lambda_i}^2 h_i^2$ appears due to time discretization. Overall, these bounds reveal a balance between contraction, target drift, and discretization.
In particular, when the sequence $\{\pi_{\lambda_i}\}$ varies slowly and the step size is chosen sufficiently small, the dynamics can effectively track the moving targets with controlled error.

Compared to the framework in \citet{habring2026forwardklconvergencetimeinhomogeneouslangevin}, our piecewise-constant results offer additional flexibility in several respects. They assume at most quadratic growth of the time derivative of the potential function to control the change of target distributions. They also assume smoothness and dissipativity with uniform parameter. In contrast, our piecewise-constant results do not require a differentiable path or impose a growth condition on $\partial_tV_t$. Instead, the change of targets is directly captured through the R\'enyi divergence term $\mathcal{R}_{\frac{24C_{\mathrm{LSI}}(\pi_{\lambda_i})}{N_ih_i}}(\pi_{\lambda_{i}}\|\pi_{\lambda_{i+1}})$, which is carried throughout the recursion. This allows us to handle more abrupt or irregular target changes, provided that the R\'enyi divergence between consecutive targets can be controlled. Moreover, both the Lipschitz constant $L_{\lambda_i}$ and the LSI constant $C_{\mathrm{LSI}}(\pi_{\lambda_i})$ are allowed to vary with time, making the result applicable to fully time-inhomogeneous settings where the geometry of the target evolves. Additionally, R\'enyi divergence is used structurally through log-Sobolev hypercontractivity and the weak triangle inequality; since $\mathcal R_\alpha$ dominates forward KL divergence for $\alpha>1$, the resulting bounds also imply forward-KL control.

In order to apply Theorem~\ref{thm:Renyi convergence for ULA with moving targets} and Theorem~\ref{thm:Renyi convergence for Langevin diffusion with moving targets}, one needs to bound the R\'enyi divergence between different distributions. We give an upper bound in the following.
\begin{theorem}
\label{thm:Renyi divergence calculation}
Let 
\[
\pi_\lambda(dx) = \frac{e^{-V_\lambda(x)}}{Z_\lambda} dx, \qquad 
\pi_\theta(dx) = \frac{e^{-V_\theta(x)}}{Z_\theta} dx,
\]
and $\delta V(x)=V_\lambda(x)-V_\theta(x)$. 
Assume the following:
\begin{itemize}
    \item $|\delta V(x)| \le \eta\,\phi(x)$ for some $\eta>0$ and a nonnegative function $\phi$.
    \item $\pi_\theta$ satisfies exponential moment conditions: for some $\tau>0$,
    \[
    \mathbb{E}_{\pi_\theta}[e^{\tau\phi}]<\infty,\quad 
    \mathbb{E}_{\pi_\theta}[\phi]<\infty,\quad 
    \mathbb{E}_{\pi_\theta}[\phi^2 e^{\tau\phi}]<\infty.
    \]
    \item Step-size condition: $\alpha\eta \le \tau$ for a given $\alpha\ge 2$.
\end{itemize}
Then the R\'enyi divergence of order $\alpha$ satisfies
\[
\mathcal{R}_\alpha(\pi_\lambda\|\pi_\theta) \le 2\alpha\,C_{\text{var}}\,\eta^2,
\]
where $C_{\text{var}} = e^{\tau\,\mathbb{E}_{\pi_\theta}[\phi]}\,\mathbb{E}_{\pi_\theta}[\phi^2 e^{\tau\phi}]$.
\end{theorem}

The proof of Theorem~\ref{thm:Renyi divergence calculation} can be found in Appendix \ref{appendix:Renyi divergence calculation}.  Note that the bound on the R\'enyi divergence depends explicitly on the difference $\delta V = V_\lambda - V_\theta$ between the two potential functions. The condition $|\delta V(x)| \le \eta \phi(x)$ ensures that the discrepancy between $V_\lambda$ and $V_\theta$ is controlled in a weighted sense by the function $\phi$, so that the R\'enyi divergence becomes small whenever the two potentials are close relative to this weight. The exponential moment assumption on $\pi_\theta$ is essential to ensure that the R\'enyi divergence is finite and that the above bound is well-defined.
The bound further reveals a quadratic dependence on the perturbation scale $\eta$, namely $\mathcal{R}_\alpha(\pi_\lambda\|\pi_\theta) =O(\alpha\eta^2)$, which reflects a second-order stability of the distribution with respect to perturbations of the potential. In many applications, the bound $|\delta V(x)| \le \eta\phi(x)$ naturally decomposes into a parameter-dependent magnitude $\eta = \eta(\lambda,\theta)$ and a fixed function $\phi$ that captures the growth in $x$. For instance, $\eta(\lambda,\theta)$ may scale proportionally to a distance between parameters (e.g., $|\lambda-\theta|$). In such cases, the above result implies that $\mathcal{R}_\alpha(\pi_\lambda\|\pi_\theta) $ is controlled by the square of the parameter difference, yielding a bound of order $O(\alpha|\lambda-\theta|^2)$ up to a constant depending only on $\pi_\theta$ through $C_{\mathrm{var}}$. 

\subsection{Example: Geometric tempering and annealed Langevin Monte Carlo}

Tempering and annealing methods introduce a path of intermediate distributions between a tractable proposal and the target distribution, and have been widely used in Monte Carlo sampling \citep{neal1998annealedimportancesampling,gelman1998simulating}. In Langevin-based sampling, this leads to moving-target dynamics in which the drift is updated along the path; such ideas appear in annealed Langevin dynamics and recent non-asymptotic analyses of annealed LMC \citep{song2019generative}. We consider the standard geometric tempering path $\pi_\beta\propto \nu^{1-\beta}\pi^\beta$, whose score is available in closed form and which has recently been analyzed for Langevin dynamics in KL divergence \citep{chehab2025provableconvergencelimitationsgeometric}. We show that this classical construction can also be analyzed directly by our Rényi moving-target framework.

Specifically, let
\[
\nu(dx)\propto e^{-U_0(x)}\,dx
\]
be a proposal distribution and let
\[
\pi(dx)\propto e^{-U_1(x)}\,dx
\]
be the target distribution. For \(\beta\in[0,1]\), consider the standard geometric tempering path
\[
\pi_\beta(dx)
=
\frac{1}{Z_\beta}
\exp\bigl( -(1-\beta)U_0(x)-\beta U_1(x) \bigr)\,dx .
\]
Thus \(\pi_0=\nu\) and \(\pi_1=\pi\). The corresponding potential is
\[
V_\beta(x)=(1-\beta)U_0(x)+\beta U_1(x),
\]
and its gradient is
\[
\nabla V_\beta(x)
=
(1-\beta)\nabla U_0(x)+\beta\nabla U_1(x).
\]

We apply the moving-target LMC scheme to this path in a piecewise-constant manner. Let
\[
0=\beta_0<\beta_1<\cdots<\beta_M=1
\]
be a tempering schedule. At stage \(i\), the target is fixed as \(\pi_{\beta_i}\), and the algorithm runs \(N_i\) ULA steps with step size \(h_i\). More precisely, starting from the current state \(X_{i,0}\), we update, for \(k=0,\ldots,N_i-1\),
\[
X_{i,k+1}
=
X_{i,k}
-h_i\bigl((1-\beta_i)\nabla U_0(X_{i,k})+\beta_i\nabla U_1(X_{i,k})\bigr)
+\sqrt{2h_i}\,\xi_{i,k},
\qquad
\xi_{i,k}\sim N(0,I_d),
\]
and then set \(X_{i+1,0}=X_{i,N_i}\). This is exactly the moving-target LMC algorithm studied above with \(\lambda_i=\beta_i\) and \(V_{\lambda_i}=V_{\beta_i}\).

The next result gives a Rényi-divergence tracking bound for this scheme. The assumption below has three components. First, the intermediate distributions are required to satisfy log-Sobolev inequalities along the path. This is natural in tempering analyses, since the intermediate distributions may have worse functional inequalities than either endpoint. Second, the potentials along the path are assumed to be smooth enough for the ULA discretization error to be controlled. Third, the energy gap \(U_1-U_0\) is required to have uniform exponential moments along the path. This last assumption allows us to control the Rényi divergence between consecutive tempered targets.

\begin{corollary}[Rényi tracking for geometrically tempered LMC]
Let \(\nu(dx)\propto e^{-U_0(x)}\,dx\) be a proposal distribution and let \(\pi(dx)\propto e^{-U_1(x)}\,dx\) be the target distribution. For \(\beta\in[0,1]\), define the geometric tempering path
\[
\pi_\beta(dx)=\frac{1}{Z_\beta}\exp\bigl( -(1-\beta)U_0(x)-\beta U_1(x) \bigr)\,dx .
\]
Let \(0=\beta_0<\beta_1<\cdots<\beta_M=1\). At stage \(i\), run \(N_i\) steps of ULA targeting \(\pi_{\beta_i}\) with step size \(h_i\), and write $\Delta t_i=N_i h_i$.

Let \(\mu_i\) denote the law of the chain at the beginning of stage \(i\), and let \(\mu_{i+1}\) denote the law after this stage.

Assume that, for each \(i\), \(\pi_{\beta_i}\) satisfies a log-Sobolev inequality with a constant $C_i\ge 1$, and \(V_{\beta_i}\) has \(L_i\)-Lipschitz gradient with \(L_i\ge 1\). Assume
\[
h_i\le \frac{1}{10000\, C_i L_i^2},
\qquad
\Delta t_i\le 2C_i\log 2 .
\]
Let
\[
\varphi(x)=|U_1(x)-U_0(x)|.
\]
Suppose that there exist \(\tau>0\) such that, 
\[
C_\Delta := \sup_{\beta \in [0,1]} e^{\tau \mathbb E_{\pi_\beta}[\varphi]} \, \mathbb E_{\pi_\beta}[\varphi^2 e^{\tau\varphi}] < \infty.
\]
If $\frac{24C_i}{\Delta t_i}(\beta_{i+1}-\beta_i)\le \tau$, then
\[
\mathcal R_2(\mu_{i+1}\|\pi_{\beta_{i+1}})
\le
e^{-\Delta t_i/(12C_i)}\mathcal R_2(\mu_i\|\pi_{\beta_i})
+400N_i d L_i^2 h_i^2
+\frac{48C_iC_\Delta}{\Delta t_i}(\beta_{i+1}-\beta_i)^2 .
\]
Consequently,
\[
\mathcal R_2(\mu_M\|\pi)
\le
\Bigl(\prod_{i=0}^{M-1}a_i\Bigr)\mathcal R_2(\mu_0\|\nu)
+
\sum_{i=0}^{M-1}
\Bigl(\prod_{j=i+1}^{M-1}a_j\Bigr)
\left[
400N_i d L_i^2 h_i^2
+\frac{48C_iC_\Delta}{\Delta t_i}(\beta_{i+1}-\beta_i)^2
\right],
\]
where
\[
a_i=\exp\bigl(-\Delta t_i/(12C_i)\bigr).
\]
\end{corollary}

\begin{proof}
For the geometric tempering path, the difference between two consecutive potentials is
\[
V_{\beta_i}(x)-V_{\beta_{i+1}}(x)
=
(\beta_i-\beta_{i+1})(U_1(x)-U_0(x)).
\]
Therefore,
\[
|V_{\beta_i}(x)-V_{\beta_{i+1}}(x)|
\le
(\beta_{i+1}-\beta_i)\varphi(x).
\]
Applying the perturbation bound for Rényi divergence with
\[
\eta_i=\beta_{i+1}-\beta_i,
\qquad
\alpha_i=\frac{24C_i}{\Delta t_i},
\]
we obtain, under the condition \(\alpha_i\eta_i\le \tau\),
\[
\mathcal R_{\alpha_i}(\pi_{\beta_i}\|\pi_{\beta_{i+1}})
\le
2\alpha_i C_\Delta(\beta_{i+1}-\beta_i)^2
=
\frac{48C_iC_\Delta}{\Delta t_i}(\beta_{i+1}-\beta_i)^2 .
\]
The one-step recursion then follows directly from the moving-target LMC bound, since during stage \(i\) the algorithm is exactly ULA targeting \(\pi_{\beta_i}\). This gives
\[
\mathcal R_2(\mu_{i+1}\|\pi_{\beta_{i+1}})
\le
e^{-\Delta t_i/(12C_i)}\mathcal R_2(\mu_i\|\pi_{\beta_i})
+400N_i d L_i^2 h_i^2
+\mathcal R_{\alpha_i}(\pi_{\beta_i}\|\pi_{\beta_{i+1}}).
\]
Substituting the previous estimate for the last term proves the first claim. The global bound follows by iterating the recursion over \(i=0,\ldots,M-1\), using \(\pi_{\beta_0}=\nu\) and \(\pi_{\beta_M}=\pi\).
\end{proof}

\begin{remark}[Uniform schedule]
Suppose that $C_i\le C$ and $L_i\le L$ for all $i$. Since each $\pi_{\beta_i}$ also satisfies an LSI with the common constant $C$, we apply the preceding corollary using the constants $C$ and $L$. Choose
\[
\beta_i=\frac{i}{M},
\qquad
\Delta t_i=\omega C,
\qquad
0<\omega\le 2\log 2,
\]
then \(a_i=e^{-\omega/12}\). The target-drift condition becomes
\[
M\ge \frac{24}{\omega\tau}.
\]
If a common step size \(h\le 1/(10000CL^2)\) is used, then
\[
400N_i dL_i^2h^2
=
400dL_i^2 h\Delta t_i
\le
400\omega C dL^2h .
\]
Consequently, if \(\mu_0=\nu\), then
\[
\mathcal R_2(\mu_M\|\pi)
\le
\frac{
400\omega C dL^2h
+
48C_\Delta/(\omega M^2)
}{
1-e^{-\omega/12}
}.
\]
This bound makes explicit the tradeoff between discretization and tempering error: the first term is controlled by the ULA step size, whereas the second term is controlled by the mesh size of the tempering schedule.
\end{remark}

\section{LMC with successive Moreau envelopes}

In this section, we apply our framework to LMC with successive Moreau envelopes introduced in \citet{habring2025diffusionabsolutezerolangevin}. 

We consider the problem of sampling from the target distribution $\pi(x)\propto e^{-V(x)}$. 
Assume that the potential $V:\mathbb{R}^d\to\mathbb{R}$ can be decomposed as $V(x)=f(x)+g(x)$ with the following assumptions:
\begin{itemize}
    \item $f$ is differentiable and $\nabla f$ is $L_f$-Lipschitz:
    \[
    \|\nabla f(x)-\nabla f(y)\|\le L_f\|x-y\|,\quad \forall x,y\in\mathbb{R}^d.
    \]
    \item $g:\mathbb R^d\to\mathbb R$ is convex and may be non-differentiable.
\end{itemize}
A key difficulty in this setting arises from the fact that the nonsmooth component $g$ is not gradient-Lipschitz and may even be non-differentiable. The main idea of using successive Moreau envelopes is to replace the nonsmooth potential $V$ with a smooth approximation. To be precise, let $g_\lambda$ be the Moreau envelope of $g$ (see Section~\ref{para:Moreau} for its definition and properties). We consider regularized potential $V_\lambda(x):=f(x)+g_\lambda(x)$ and corresponding distribution $\pi_\lambda(x)\propto e^{-V_\lambda(x)}$. Then $V_\lambda(x)$ is differentiable with gradient $\nabla V_\lambda(x)=\nabla f(x)+\frac{1}{\lambda}\bigl(x-\operatorname{prox}_{\lambda g}(x)\bigr)$ and $\nabla V_\lambda(x)$ is $L_{V_\lambda}$-Lipschitz continuous with $L_{V_\lambda}\le L_f+\frac{1}{\lambda}$. Consequently, the original nonsmooth potential $V$ can be approximated by the smooth surrogate $V_\lambda$, which enables the use of gradient-based analysis and sampling algorithms.
For convenience, we set
\[
g_0=g,
\qquad
V_0=V,
\qquad
\pi_0=\pi.
\]
All Langevin updates below use strictly positive Moreau parameters.
The endpoint $\lambda=0$ is used only to measure the approximation
error with respect to the original target.

\subsection{R\'enyi divergence between the $\pi_\lambda$}

In order to apply our framework in the previous section, we need to bound the R\'enyi divergence between different $\pi_\lambda$. By Theorem~\ref{thm:Renyi divergence calculation}, we need first to bound the difference between different $V_\lambda(x)$. We analyze this through the explicit expression of $\frac{\partial g_\lambda}{\partial\lambda}(x)$ (see Section~\ref{para:Moreau}).

Generally, $\frac{\partial g_\lambda}{\partial\lambda}(x)$ is bounded by $O(\frac{1}{\lambda})$ and we have the following bound.

\begin{lemma}
    If $g$ is convex with $\inf g>-\infty$, then $|g_{\lambda_1}(x)-g_{\lambda_2}(x)|\le 2\bigl(g(x)-\inf g\bigr)\left|\ln\frac{\lambda_1}{\lambda_2}\right|$ for any $\lambda_1,\lambda_2>0$.
\end{lemma}
\begin{proof}

From the optimality condition, $\frac{1}{2\lambda}\|x-\operatorname{prox}_{\lambda g}(x)\|^2+g(\operatorname{prox}_{\lambda g}(x))\le g(x)$. Thus, 
\[
\frac{1}{2\lambda}\|x-\operatorname{prox}_{\lambda g}(x)\|^2\le g(x)-g(\operatorname{prox}_{\lambda g}(x))\le g(x)-\inf g.
\]
Therefore
\[
\left|\frac{\partial g_\lambda}{\partial\lambda}(x)\right|
=\frac{1}{2\lambda^2}\|x-\operatorname{prox}_{\lambda g}(x)\|^2
\le\frac{g(x)-\inf g}{\lambda}.
\]
Integrating the above expression with respect to $\lambda$ yields the conclusion.
\end{proof}

It's natural to ask whether $\frac{\partial g_\lambda}{\partial\lambda}(x)$ have an upper bound which doesn't depend on $\lambda$. We give a condition in the following. 

\begin{lemma}
\label{lem:moreau-lambda-stability}
If $g$ is convex with $\inf g>-\infty$. Let $\Phi:[0,\infty)\to[1,\infty)$ be a nondecreasing function. If there exists $C>0$ such that for any $x\in\mathbb{R}^d$ and $h\in\partial g(x)$,
\[
\|h\|^2\le C\Phi(g(x)-\inf g),
\]
then for any $\lambda_1,\lambda_2>0$,
\[
|g_{\lambda_1}(x)-g_{\lambda_2}(x)|
\le
\frac{C}{2}\Phi(g(x)-\inf g)|\lambda_1-\lambda_2|.
\]
\end{lemma}

\begin{proof}
From the optimality condition,
\[
0\in\partial g(\operatorname{prox}_{\lambda g}(x))
+\frac{1}{\lambda}(\operatorname{prox}_{\lambda g}(x)-x).
\]
Equivalently,
\[
\frac{1}{\lambda}(x-\operatorname{prox}_{\lambda g}(x))
\in
\partial g(\operatorname{prox}_{\lambda g}(x)).
\]
By the assumption,
\[
\|
\frac{1}{\lambda}(x-\operatorname{prox}_{\lambda g}(x))
\|^2
\le
C\Phi(g(\operatorname{prox}_{\lambda g}(x))-\inf g).
\]
Moreover, by the definition of the proximal mapping,
\[
g(\operatorname{prox}_{\lambda g}(x))
+
\frac{1}{2\lambda}
\|x-\operatorname{prox}_{\lambda g}(x)\|^2
\le
g(x),
\]
and hence
\[
g(\operatorname{prox}_{\lambda g}(x))\le g(x).
\]
Since $\Phi$ is nondecreasing, we obtain
\[
\|
\frac{1}{\lambda}(x-\operatorname{prox}_{\lambda g}(x))
\|^2
\le
C\Phi(g(x)-\inf g).
\]
Therefore,
\[
\left|\frac{\partial g_\lambda}{\partial\lambda}(x)\right|
=
\frac{1}{2\lambda^2}
\|x-\operatorname{prox}_{\lambda g}(x)\|^2
\le
\frac{C}{2}\Phi(g(x)-\inf g).
\]
Integrating the above expression with respect to $\lambda$ yields the conclusion.
\end{proof}

As an illustration, we discuss the case where the conditions of the above lemma are satisfied. To bound the R\'enyi divergence between $\pi_{\lambda_1}$ and $\pi_{\lambda_2}$, we can choose $\eta=\left|\lambda_1-\lambda_2\right|$ and $\phi=\frac{1}{2}C\Phi(g(x)-\inf g)$ in Theorem~\ref{thm:Renyi divergence calculation}. If the corresponding exponential moment conditions in Theorem~\ref{thm:Renyi divergence calculation} hold under \(\pi_{\lambda_2}\), then
\[
\mathcal R_\alpha(\pi_{\lambda_1}\|\pi_{\lambda_2})
\le
C' \alpha |\lambda_1-\lambda_2|^2
\]
whenever \(\alpha|\lambda_1-\lambda_2|\) is sufficiently small. In some applications, these exponential moment bounds can be chosen uniformly for \(\lambda\) in a bounded interval. In this case, the above estimate gives a uniform quadratic stability bound for the Moreau-regularized targets.

\begin{lemma}[Extension to the endpoint]
\label{lem:renyi-stability-endpoint}
Suppose that there exist $C_{\text{R\'enyi}}>0$ and $\tau>0$ such
that
\[
\mathcal R_\alpha(\pi_\lambda\|\pi_\theta)
\le
\alpha C_{\text{R\'enyi}}(\lambda-\theta)^2
\]
for all $0<\lambda,\theta\le\lambda_{\rm max}, \alpha\ge2, \alpha|\lambda-\theta|\le\tau$. Then, for every $0<\lambda\le\lambda_{\rm max}$ and $\alpha\ge2$ satisfying $\alpha\lambda\le\tau$,
\[
\mathcal R_\alpha(\pi_\lambda\|\pi_0)
\le
\alpha C_{\text{R\'enyi}}\lambda^2.
\]
\end{lemma}

\begin{proof}
Fix $\theta_0\in(0,\lambda_{\rm max}]$. Since $g$ is proper, lower
semicontinuous, and convex,
\[
g_\theta\uparrow g
\qquad\text{pointwise as }\theta\downarrow0,
\]
and, for $0<\theta\le\theta_0$,
\[
g_{\theta_0}\le g_\theta\le g.
\]
Hence, by dominated convergence,
\[
Z_\theta\longrightarrow Z_0,
\]
and the corresponding Lebesgue densities $p_\theta$ converge
pointwise to $p_0$.

For fixed $\lambda>0$, Fatou's lemma gives
\[
\begin{aligned}
\exp\left(
(\alpha-1)\mathcal R_\alpha(\pi_\lambda\|\pi_0)
\right)
&=
\int p_\lambda^\alpha p_0^{1-\alpha}\,dx
\\
&\le
\liminf_{\theta\downarrow0}
\int p_\lambda^\alpha p_\theta^{1-\alpha}\,dx
\\
&=
\liminf_{\theta\downarrow0}
\exp\left(
(\alpha-1)
\mathcal R_\alpha(\pi_\lambda\|\pi_\theta)
\right).
\end{aligned}
\]
Taking $\theta\downarrow0$ with $\theta<\lambda$ and using
$\alpha\lambda\le\tau$ proves the result.
\end{proof}

\subsection{Convergence analysis of the algorithm}

\begin{algorithm}[htbp]
\caption{ULA using successive Moreau Envelopes}
\label{alg:ULA using successive Moreau Envelopes}
\begin{algorithmic}[1]
\Require Initial distribution $\hat\mu_0$, number of steps $\hat N$, $N$, 
        step size $h$ for first phase, step sizes $\{h_n\}_{n=1}^N$, 
        potential functions  $\{V_{\lambda_n}\}_{n=1}^N$
\Ensure Final sample $X_N$

\State Sample $\hat X_0 \sim \hat\mu_0$

\For{$k = 0$ to $\hat N - 1$}
    \State $\hat X_{k+1} \gets \hat X_k - h \nabla V_{\lambda_1}(\hat X_k) + \sqrt{2h} \, \xi_k$, 
          \Comment{$\xi_k \sim \mathcal{N}(0, I_d)$}
\EndFor

\State $X_0 \gets \hat X_{\hat N}$ \Comment{Intermediate sample $\mu_0$}

\For{$n = 1$ to $N$}
    \State $X_n \gets X_{n-1} - h_n \nabla V_{\lambda_n}(X_{n-1}) + \sqrt{2h_n} \, \zeta_n$, 
          \Comment{$\zeta_n \sim \mathcal{N}(0, I_d)$}
\EndFor

\State \Return $X_N$
\end{algorithmic}
\end{algorithm}

We consider the multi-level Langevin scheme introduced in \citet{habring2025diffusionabsolutezerolangevin}, where sampling is performed over a sequence of Moreau envelopes with decreasing smoothing parameters. At each level, a Langevin dynamics associated with a fixed Moreau-regularized potential is run for several iterations, and the output is passed to the next level.

While this framework provides an effective sampling strategy, a precise non-asymptotic characterization of the algorithm, in particular with respect to the choice of step sizes and smoothing parameters, remains unclear.

In this section, we study a concrete instantiation of this scheme by specifying a sequence of potentials $\{V_{\lambda_n}\}_{n=1}^N$ together with corresponding step sizes $\{h_n\}_{n=1}^N$. Starting from an initial distribution obtained by running ULA targeting $V_{\lambda_1}$, we then perform successive ULA updates, where at step $n$ we apply one update targeting $V_{\lambda_n}$. The full procedure is described in Algorithm~\ref{alg:ULA using successive Moreau Envelopes}.

We provide a convergence guarantee for the algorithm in the following theorem, with explicit choice of step sizes and smoothing parameters.
\begin{theorem}
\label{thm:Convergence of ULA using successive Moreau Envelopes}
Consider Algorithm~\ref{alg:ULA using successive Moreau Envelopes}.
Let $\hat\mu_0$ denote the initial distribution and $\mu_n$ denote
the law of $X_n$.
Assume the following conditions hold:
\begin{itemize}
    \item The distribution
    $\pi_{\lambda}\propto e^{-V_{\lambda}}$ on $\mathbb{R}^d$
    satisfies a log-Sobolev inequality (LSI) with uniform constant
    $C_{\mathrm{LSI}}\ge1$ for
    $0<\lambda\le\lambda_{\rm max}$.

    \item For every $\lambda>0$, the gradient $\nabla V_{\lambda}$ is
    $(L_f+\frac{1}{\lambda})$-Lipschitz with $L_f\ge1$.

    \item The R\'enyi divergence
    between different $\pi_\lambda$ satisfies
    \[
    \mathcal R_\alpha(\pi_\lambda\|\pi_\theta)
    \le
    \alpha C_{\text{R\'enyi}}(\lambda-\theta)^2,\qquad
    \forall
    0<\lambda,\theta\le\lambda_{\rm max},
    \alpha|\lambda-\theta|\le\tau,
    \alpha\ge2.
    \]
    By Lemma~\ref{lem:renyi-stability-endpoint}, the same estimate holds with $\theta=0$ whenever $\alpha\lambda\le\tau$.
\end{itemize}
Fix $\gamma>0$ and define $n_0=
16\gamma^4L_f^4
+
240000^2C_{\mathrm{LSI}}^4\gamma^{-4}
+
256\gamma^4\tau^{-4}
+
16\gamma^4\lambda_{\rm max}^{-4}$.
Choose $h_n=\frac{24C_{\mathrm{LSI}}}{n+n_0},\lambda_n
=
\left(
\frac{(n+n_0)^{1/4}}{\gamma}-L_f
\right)^{-1}$.
For $\epsilon\in(0,1)$, choose $h
=
\Theta
\left(
\frac{
\epsilon
}{
dC_{\mathrm{LSI}}
\left(
L_f+C_{\mathrm{LSI}}\gamma^{-2}
+\tau^{-1}+\lambda_{\rm max}^{-1}
\right)^2
}
\right)$.
Then, for {\small $\hat N\ge\Omega\left(\frac{dC_{\mathrm{LSI}}^2\left(L_f+C_{\mathrm{LSI}}\gamma^{-2}+\tau^{-1}+\lambda_{\rm max}^{-1}\right)^2}{\epsilon}\log\left(e+\frac{\mathcal R_2(\hat\mu_0\|\pi_{\lambda_1})}{\epsilon}\right)\right)$}
and {\small $N\ge\Omega\left(\frac{d^2C_{\mathrm{LSI}}^4\gamma^{-4}+C_{\text{R\'enyi}}^2\gamma^4}{\epsilon^2}\right)$},
we have
\[
\mathcal R_{4/3}(\mu_N\|\pi)\le\epsilon.
\]

In particular, choosing $\gamma^4=\frac{dC_{\mathrm{LSI}}^2}{C_{\text{R\'enyi}}}$ balances the two terms in the continuation complexity. Thus it suffices to take
\[
N=O\left(\frac{dC_{\mathrm{LSI}}^2C_{\text{R\'enyi}}}{\epsilon^2}\right).
\]
With this choice, $\hat N$ and $N$ can be chosen so that the total number of updates satisfies
\[
\hat N+N
=O\left(\frac{dC_{\mathrm{LSI}}^2C_{\text{R\'enyi}}}{\epsilon^2}\right)+ O\Bigg(\frac{dC_{\mathrm{LSI}}^2
}{
\epsilon
}
\left(
L_f
+
\sqrt{
\frac{
C_{\text{R\'enyi}}
}{
d
}
}
+
\tau^{-1}
+
\lambda_{\rm max}^{-1}
\right)^2
\log\left(
e+
\frac{
\mathcal R_2(\hat\mu_0\|\pi_{\lambda_1})
}{
\epsilon
}
\right)
\Bigg).
\]
\end{theorem}

The proof of Theorem~\ref{thm:Convergence of ULA using successive Moreau Envelopes} can be found in Appendix \ref{appendix:Convergence of ULA using successive Moreau Envelopes}. 

Note that our smoothness bound $L_{V_\lambda}\le L_f+\frac{1}{\lambda}$ explicitly tracks the deterioration of the Moreau-regularized potential as $\lambda\downarrow0$. The Moreau specialization of \citet{habring2026forwardklconvergencetimeinhomogeneouslangevin} instead starts from an original potential that is differentiable and has a Lipschitz and dissipative gradient, and works on a parameter interval where the regularized smoothness is uniformly controlled. Our setting therefore directly accommodates a decomposition $V=f+g$ with a nonsmooth convex component $g$. We also assume that $\pi_\lambda \propto e^{-V_\lambda}$ satisfies a log-Sobolev inequality with a uniform constant $C_{\mathrm{LSI}}$ for $\lambda \le \lambda_{\max}$. This is mainly for simplicity. When $\lambda$ ranges over a compact interval, $g_\lambda$ can be viewed as a smooth approximation of $g$ with controlled perturbation, so it is plausible in many settings that the corresponding measures $\pi_\lambda$ share comparable functional inequality constants. Moreover, the R\'enyi divergence condition quantifies how fast the target distributions $\pi_\lambda$ change as $\lambda$ varies, and is consistent with the discussion at the end of the previous subsection.

The parameter choices in the theorem reflect a balance between two competing effects. As $\lambda_n$ decreases, the target distribution $\pi_{\lambda_n}$ becomes closer to the original distribution $\pi$, while the smoothness of $V_{\lambda_n}$ deteriorates since the Lipschitz constant grows as $1/\lambda_n$. Consequently, the step size $h_n$ must decrease to control the discretization error of ULA. At the same time, if $\lambda_n$ changes too quickly, the “moving target” error (i.e., the mismatch between $\pi_{\lambda_{n-1}}$ and $\pi_{\lambda_n}$) becomes large. The schedules ${\lambda_n}$ and ${h_n}$ are therefore designed to balance discretization error and target drift, leading to the overall convergence guarantee in R\'enyi divergence.

\subsection{Comparison with fixed-$\lambda$ MYULA}
\label{subsec:fixed-versus-successive-moreau}

We compare the successive scheme with fixed-$\lambda$ MYULA under
the same assumptions and the same final
$\mathcal R_{4/3}$-accuracy criterion.

\paragraph{Fixed-$\lambda$ MYULA.}
Assume that $\epsilon$ is sufficiently small so that
\[
\epsilon
\le
\min\left\{
1,\,
4C_{\text{R\'enyi}}\lambda_{\rm max}^2,\,
C_{\text{R\'enyi}}\tau^2
\right\}.
\]
Choose
\[
\lambda_{\rm fix}
=
\sqrt{
\frac{\epsilon}{4C_{\text{R\'enyi}}}
}.
\]
Then $\lambda_{\rm fix}\le\lambda_{\rm max}, 2\lambda_{\rm fix}\le\tau$, and the R\'enyi-stability assumption gives
\[
\mathcal R_2(\pi_{\lambda_{\rm fix}}\|\pi)
\le
2C_{\text{R\'enyi}}\lambda_{\rm fix}^2
=
\frac{\epsilon}{2}.
\]

Applying
\citet[Theorem~4]{chewi2024analysislangevinmontecarlo}
to the fixed target $\pi_{\lambda_{\rm fix}}$, with R\'enyi order
$q=3$ and target accuracy $\epsilon/4$, gives
\[
\mathcal R_2(\mu_K\|\pi_{\lambda_{\rm fix}})
\le
\mathcal R_3(\mu_K\|\pi_{\lambda_{\rm fix}})
\le
\frac{\epsilon}{4}
\]
after
\[
\begin{aligned}
K_{\rm fix}
=
O\Bigg(
dC_{\mathrm{LSI}}^2
\left(
\frac{C_{\text{R\'enyi}}}{\epsilon^2}
+
\frac{L_f^2}{\epsilon}
\right)
\log\left(
e+
\frac{
\mathcal R_2(
\hat\mu_0\|\pi_{\lambda_{\rm fix}}
)
}{
\epsilon
}
\right)
\Bigg)
\end{aligned}
\]
iterations. Here, we used
$\lambda_{\rm fix}^{-1}
=2\sqrt{C_{\text{R\'enyi}}/\epsilon}$, and the logarithmic factor
is obtained by solving the exponentially decaying initialization
term in the error bound of that theorem at accuracy $\epsilon$.
By the weak triangle inequality,
\[
\begin{aligned}
\mathcal R_{4/3}(\mu_K\|\pi)
\le
2\mathcal R_2(\mu_K\|\pi_{\lambda_{\rm fix}})
+
\mathcal R_2(\pi_{\lambda_{\rm fix}}\|\pi)
\le
\epsilon.
\end{aligned}
\]

\paragraph{Balanced successive-Moreau schedule.}
For the balanced choice
\[
\gamma^4
=
\frac{
dC_{\mathrm{LSI}}^2
}{
C_{\text{R\'enyi}}
},
\]
Theorem~\ref{thm:Convergence of ULA using successive Moreau Envelopes}
gives
\[
K_{\rm succ}
={}
O\left(
\frac{
dC_{\mathrm{LSI}}^2C_{\text{R\'enyi}}
}{
\epsilon^2
}
\right)
+
O\Bigg(
\frac{
dC_{\mathrm{LSI}}^2
}{
\epsilon
}
\left(
L_f
+
\sqrt{
\frac{
C_{\text{R\'enyi}}
}{
d
}
}
+
\tau^{-1}
+
\lambda_{\rm max}^{-1}
\right)^2
\log\left(
e+
\frac{
\mathcal R_2(
\hat\mu_0\|\pi_{\lambda_1}
)
}{
\epsilon
}
\right)
\Bigg).
\]

Thus, within the present upper-bound analysis, the two schemes have
the same leading polynomial dependence
\[
\frac{
dC_{\mathrm{LSI}}^2C_{\text{R\'enyi}}
}{
\epsilon^2
}.
\]
The difference is where the initializer-dependent factor appears.
For fixed-$\lambda$ MYULA, this factor multiplies the leading
high-accuracy term. For the successive scheme, it appears only in
the first-phase cost and does not multiply the leading continuation
cost.

Fixed-$\lambda$ MYULA contracts the initial error while already
targeting
\[
\lambda_{\rm fix}\asymp\sqrt{\epsilon},
\]
whose smoothness constant grows as $\epsilon^{-1/2}$. The successive
scheme instead reduces the initial error at the fixed target
$\pi_{\lambda_1}$ and then gradually decreases the Moreau parameter
using warm starts from the preceding targets.

For fixed problem parameters and as $\epsilon\downarrow0$, the
first-phase term of the successive bound is of lower order in
$\epsilon$ than its continuation term. This is an upper-bound
comparison rather than a claim of unconditional practical
superiority: the constants are conservative, and the two
initialization divergences in the bounds need not be directly
comparable.

\subsection{Verifying the assumptions: a nonconvex example class}

Theorem~\ref{thm:Convergence of ULA using successive Moreau Envelopes} is stated under three assumptions: a uniform log-Sobolev inequality for the Moreau-regularized targets, a smoothness bound for \(V_\lambda\), and a Rényi stability bound between \(\pi_\lambda\) and \(\pi_\theta\). We now give a class of potentials for which these assumptions can be verified explicitly.

The example below shows that the framework is not restricted to globally convex potentials. We allow the potential to contain a bounded smooth perturbation \(W\), which may make the full potential nonconvex. The uniform LSI follows from the strongly convex component and the Holley--Stroock perturbation principle, while the Rényi stability follows from the bounded-subgradient property of the nonsmooth component.

\begin{proposition}[A nonconvex example class]\label{prop:nonconvex-example-class}
Let
\[
V(x)=f(x)+W(x)+g(x),
\qquad
V_\lambda(x)=f(x)+W(x)+g_\lambda(x),
\]
and let
\[
\pi_\lambda(dx)=Z_\lambda^{-1}e^{-V_\lambda(x)}\,dx,
\qquad
\pi(dx)=Z^{-1}e^{-V(x)}\,dx.
\]
Assume that the following conditions hold:
\begin{enumerate}
\item $f$ is differentiable, $m$-strongly convex, and $\nabla f$ is $L_f$-Lipschitz.
\item $W$ is differentiable, $\nabla W$ is $L_W$-Lipschitz, and
\[
\operatorname{osc}(W):=\sup_{x\in\mathbb{R}^d}W(x)-\inf_{x\in\mathbb{R}^d}W(x)<\infty.
\]
\item $g:\mathbb R^d\to\mathbb R$ is convex with $\inf g>-\infty$, and its subgradients are uniformly bounded: there exists $G>0$ such that
\[
\|h\|\le G,\qquad \forall x\in\mathbb{R}^d,\ h\in\partial g(x).
\]
\end{enumerate}
Then the family $\{\pi_\lambda\}_{\lambda>0}$ satisfies the following properties.
First, $\pi_\lambda$ satisfies a log-Sobolev inequality with a constant uniform in $\lambda$:
\[
C_{\mathrm{LSI}}(\pi_\lambda)
\le
\frac{e^{\operatorname{osc}(W)}}{m}.
\]
Second, $\nabla V_\lambda$ is Lipschitz with
\[
L_\lambda
\le
L_f+L_W+\frac1\lambda.
\]
Third, for any $\lambda,\theta\ge0$ and any $\alpha\ge2$ satisfying
\[
\alpha|\lambda-\theta|\le \frac{2}{G^2},
\]
where we use the convention $g_0=g$ and $\pi_0=\pi$, we have
\[
\mathcal{R}_\alpha(\pi_\lambda\|\pi_\theta)
\le
\frac{e^2G^4}{2}\alpha(\lambda-\theta)^2.
\]
In particular, the Rényi stability condition in Theorem~\ref{thm:Convergence of ULA using successive Moreau Envelopes} holds with
\[
C_{\text{R\'enyi}}=\frac{e^2G^4}{2},
\qquad
\tau=\frac{2}{G^2}.
\]
\end{proposition}

\begin{proof}
We first prove the uniform log-Sobolev inequality. Since $f$ is $m$-strongly convex and $g_\lambda$ is convex, the function $f+g_\lambda$ is $m$-strongly convex. Hence the probability measure
\[
\widetilde\pi_\lambda(dx)
=
\widetilde Z_\lambda^{-1}e^{-f(x)-g_\lambda(x)}\,dx
\]
satisfies a log-Sobolev inequality with constant at most $1/m$. Since $\pi_\lambda$ is obtained from $\widetilde\pi_\lambda$ by the bounded perturbation $W$, the Holley--Stroock perturbation principle gives
\[
C_{\mathrm{LSI}}(\pi_\lambda)
\le
e^{\operatorname{osc}(W)}C_{\mathrm{LSI}}(\widetilde\pi_\lambda)
\le
\frac{e^{\operatorname{osc}(W)}}{m}.
\]
This bound is uniform in $\lambda$.

Next, since $\nabla f$ is $L_f$-Lipschitz, $\nabla W$ is $L_W$-Lipschitz, and $g_\lambda$ is $1/\lambda$-smooth, we have
\[
\|\nabla V_\lambda(x)-\nabla V_\lambda(y)\|
\le
\left(L_f+L_W+\frac1\lambda\right)\|x-y\|,
\]
which proves the smoothness bound.

It remains to control the Rényi divergence between $\pi_\lambda$ and $\pi_\theta$. By the boundedness of the subgradients of $g$, Lemma~\ref{lem:moreau-lambda-stability} applies with $\Phi\equiv1$ and $C=G^2$. Therefore, for any $\lambda,\theta>0$,
\[
|g_\lambda(x)-g_\theta(x)|
\le
\frac{G^2}{2}|\lambda-\theta|.
\]
The same bound also holds when one of the parameters is zero, with the convention $g_0=g$. Indeed, since $g$ is $G$-Lipschitz, the Moreau envelope satisfies
\[
0\le g(x)-g_\lambda(x)\le \frac{G^2}{2}\lambda.
\]
Consequently, for all $\lambda,\theta\ge0$,
\[
|V_\lambda(x)-V_\theta(x)|
=
|g_\lambda(x)-g_\theta(x)|
\le
\frac{G^2}{2}|\lambda-\theta|.
\]
Applying Theorem~\ref{thm:Renyi divergence calculation} with
\[
\eta=\frac{G^2}{2}|\lambda-\theta|,
\qquad
\varphi\equiv1,
\qquad
\tau=1,
\]
we obtain, whenever $\alpha\eta\le1$, equivalently $\alpha|\lambda-\theta|\le 2/G^2$,
\[
\mathcal{R}_\alpha(\pi_\lambda\|\pi_\theta)
\le
2\alpha e^2 \eta^2
=
\frac{e^2G^4}{2}\alpha(\lambda-\theta)^2.
\]
This completes the proof.
\end{proof}

To apply Theorem~\ref{thm:Convergence of ULA using successive Moreau Envelopes}, we use the decomposition
\[
V=\widetilde f+g,
\qquad
\widetilde f:=f+W.
\]
The smoothness and LSI constants in that theorem may be chosen as
\[
\overline L_f
:=
\max\{1,L_f+L_W\},
\qquad
\overline C_{\mathrm{LSI}}
:=
\max\left\{
1,\frac{e^{\operatorname{osc}(W)}}{m}
\right\}.
\]
The remaining constants are
\[
C_{\text{R\'enyi}}
=
\frac{e^2G^4}{2},
\qquad
\tau
=
\frac{2}{G^2}.
\]
Thus, Theorem~\ref{thm:Convergence of ULA using successive Moreau Envelopes} applies with $\overline L_f$ and $\overline C_{\mathrm{LSI}}$ in place of $L_f$ and $C_{\mathrm{LSI}}$, respectively.

We now give several common choices of $g$ for which the bounded-subgradient condition in Proposition~\ref{prop:nonconvex-example-class} is easy to verify. In all examples below, $g$ should be understood as the nonsmooth component in the decomposition $V=f+W+g$, where $f$ and $W$ satisfy the assumptions of Proposition~\ref{prop:nonconvex-example-class}. In the standard convex setting one may simply take $W=0$. More generally, the same examples remain covered after adding a smooth bounded perturbation $W$ satisfying the assumptions of Proposition~\ref{prop:nonconvex-example-class}; in this case the resulting potential may be nonconvex, while the uniform LSI, smoothness, and Rényi stability estimates remain valid.

\paragraph{Sparse and generalized lasso regularization.}
Let
\[
g(x)=\gamma\|Dx\|_1,
\]
where $D\in\mathbb{R}^{r\times d}$ is a linear operator. This class includes the usual sparse penalty $g(x)=\gamma\|x\|_1$ by taking $D=I$, and also generalized lasso and fused-lasso type penalties \citep{tibshirani1996regression}. For any $h\in\partial g(x)$, there exists $s\in\mathbb{R}^r$ with $\|s\|_\infty\le1$ such that
\[
h=\gamma D^\top s.
\]
Hence
\[
\|h\|
\le
\gamma\|D\|_{\mathrm{op}}\sqrt r.
\]
Thus Proposition~\ref{prop:nonconvex-example-class} applies with
\[
G=\gamma\|D\|_{\mathrm{op}}\sqrt r.
\]
In particular, for the sparse penalty $g(x)=\gamma\|x\|_1$, one may take
\[
G=\gamma\sqrt d.
\]

\paragraph{Total-variation regularization.}
Discrete total-variation penalties are also covered by the previous example \citep{rudin1992nonlinear}. Indeed, if $D$ is a finite-difference operator, then
\[
g(x)=\gamma\|Dx\|_1
\]
corresponds to anisotropic total-variation regularization. Therefore the proposition applies with
\[
G=\gamma\|D\|_{\mathrm{op}}\sqrt r.
\]
This provides a common example from imaging inverse problems where the nonsmooth term naturally arises from edge-preserving regularization.

\paragraph{Group-sparse regularization.}
Let
\[
g(x)=\gamma\sum_{j=1}^J\|x_{G_j}\|_2,
\]
where $\{G_j\}_{j=1}^J$ are disjoint groups \citep{yuan2006model}. Then every $h\in\partial g(x)$ satisfies
\[
\|h\|^2
=
\sum_{j=1}^J\|h_{G_j}\|^2
\le
\gamma^2J.
\]
Hence Proposition~\ref{prop:nonconvex-example-class} applies with
\[
G=\gamma\sqrt J.
\]

\paragraph{Hinge-loss type potentials.}
Let
\[
g(x)=\gamma\sum_{i=1}^n (1-y_i a_i^\top x)_+,
\]
where $a_i\in\mathbb{R}^d$ and $y_i\in\{-1,1\}$. This corresponds to the nonsmooth loss used in support-vector machines \citep{cortes1995support}. Then any $h\in\partial g(x)$ can be written as
\[
h=-\gamma\sum_{i=1}^n s_i y_i a_i,
\qquad
s_i\in[0,1].
\]
If $A\in\mathbb{R}^{n\times d}$ has rows $a_i^\top$, then
\[
\|h\|
\le
\gamma\|A\|_{\mathrm{op}}\sqrt n.
\]
Thus the proposition applies with
\[
G=\gamma\|A\|_{\mathrm{op}}\sqrt n.
\]

\paragraph{Nuclear-norm regularization.}
For matrix-valued variables $X\in\mathbb{R}^{d_1\times d_2}$, let
\[
g(X)=\gamma\|X\|_*,
\]
where $\|\cdot\|_*$ denotes the nuclear norm. This is commonly used in convex approaches to low-rank matrix recovery and matrix completion \citep{candes2012exact}. Since every subgradient of the nuclear norm has spectral norm at most one, its Frobenius norm is at most $\sqrt{\min(d_1,d_2)}$. Hence Proposition~\ref{prop:nonconvex-example-class} applies with
\[
G=\gamma\sqrt{\min(d_1,d_2)}.
\]
This covers low-rank matrix regularization problems such as matrix completion.

These examples show that the assumptions in Proposition~\ref{prop:nonconvex-example-class} cover several standard nonsmooth regularizers. The role of the proposition is to connect these concrete models with the abstract assumptions of Theorem~\ref{thm:Convergence of ULA using successive Moreau Envelopes}: the strongly convex component provides a uniform LSI, the Moreau envelope gives the required smoothness bound, and the bounded-subgradient property yields the Rényi stability between consecutive Moreau-regularized targets. Therefore, Theorem~\ref{thm:Convergence of ULA using successive Moreau Envelopes} applies to these sampling problems with explicit choices of the constants \(C_{\mathrm{LSI}}\), \(L_\lambda\), \(C_{\text{R\'enyi}}\), and \(\tau\).

\section{Conclusion}
In this paper, we have analyzed the convergence of LMC with moving targets in R\'enyi divergence. Our analysis is based on the approach developed in \citet{chewi2024analysislangevinmontecarlo} to bound the R\'enyi divergence of LMC and the weak triangle inequality. Then we have analyzed LMC with successive Moreau envelopes proposed by \citet{habring2025diffusionabsolutezerolangevin}. We have provided precise choice of step sizes and smoothing parameters.


\newpage

\bibliographystyle{plainnat}
\bibliography{reference}

\newpage

\appendix

\section{Technical appendices and supplementary material}

\subsection{Proof of \eqref{eq:upper bound for D}}\label{subsec:upper bound for D}

Since $\nabla V$ is $L$-Lipschitz, we have 
\[\|\nabla V(x_{kh})\| 
\le \|\nabla V(x_t)\| + L \|x_t - x_{kh}\| \\
\le \|\nabla V(x_t)\| + hL \|\nabla V(x_{kh})\| + \sqrt{2}L \|B_t-B_{kh}\|.
\]
Since $h \leq \frac{1}{3L}$, then rearranging gives $\|\nabla V(x_{kh})\| \leq \frac{3}{2} \|\nabla V(x_t)\| + \frac{3\sqrt{2}L}{2} \|B_t-B_{kh}\|$.
Thus we have
\begin{align*}
\|\nabla V(x_t) - \nabla V(x_{kh})\|^2 
&\le L^2\|x_t - x_{kh}\|^2\\
&\le 2L^2 (t-kh)^2 \|\nabla V(x_{kh})\|^2 + 4L^2 \|B_t-B_{kh}\|^2\\
&\le 9L^2 (t-kh)^2 \|\nabla V(x_{t})\|^2 + (18L^4(t-kh)^2+4L^2) \|B_t-B_{kh}\|^2.
\end{align*}

Thus
\begin{align*}
D(t) &\le 9L^2 (t-kh)^2\mathbb{E}[\psi_t(x_t) \|\nabla V(x_{t})\|^2]+(18L^4(t-kh)^2+4L^2)\mathbb{E}[\psi_t(x_t) \|B_t-B_{kh}\|^2]\\
&\le 9L^2 (t-kh)^2\mathbb{E}[\psi_t(x_t) \|\nabla V(x_{t})\|^2]+6L^2\mathbb{E}[\psi_t(x_t) \|B_t-B_{kh}\|^2].
\end{align*}

For the first term, we apply Lemma~20 in \citet{chewi2024analysislangevinmontecarlo} to the measure $\psi_t\mu_t$:
\begin{align*}
\mathbb{E}_{\psi_t\mu_t}\bigl[\|\nabla V\|^2\bigr]
&\leq \mathbb{E}_{\mu_t}\left[ \psi_t \left\| \nabla \ln (\psi_t \frac{d\mu_t}{d\pi}) \right\|^2 \right] + 2dL \\
&= \frac{\mathbb{E}_\pi\bigl[\rho_t^{q(t)} \,\|\nabla\ln(\rho_t^{q(t)})\|^2\bigr]}{\mathbb{E}_\pi(\rho_t^{q(t)})} + 2dL \\
&= \frac{4\,\mathbb{E}_\pi\bigl[\|\nabla(\rho_t^{q(t)/2})\|^2\bigr]}{\mathbb{E}_\pi(\rho_t^{q(t)})} + 2dL .
\end{align*}

For the second term, we apply Lemma~19 in \citet{chewi2024analysislangevinmontecarlo}:
\[
\mathbb{E}\bigl[\psi_t(x_t)\,\|B_t-B_{kh}\|^2\bigr]
\leq 14d(t-kh) + 32hC_{\mathrm{LSI}}\,\frac{\mathbb{E}_\pi\bigl[\|\nabla(\rho_t^{q(t)/2})\|^2\bigr]}{\mathbb{E}_\pi(\rho_t^{q(t)})} .
\]
Thus we have proved \eqref{eq:upper bound for D}.

\subsection{General R\'enyi orders}
\label{appendix:general-renyi-orders}

We first extend
Lemma~\ref{lem:Renyi convergence for ULA}
to an arbitrary initial R\'enyi order.

\begin{lemma}[General-order short-time LMC bound]
\label{lem:general-order-ula}
Let $\mu_t$ be the law of the continuous interpolation of LMC
targeting
\[
\pi\propto e^{-V}
\]
on $\mathbb R^d$. Assume that $\pi$ satisfies an LSI with constant
$C\ge1$ and that $\nabla V$ is $L$-Lipschitz with $L\ge1$.

Fix $p\ge2$, let
\[
m_p=2p-1,
\qquad
T=Nh,
\]
and assume
\[
T\le2C\log2,
\qquad
h
\le
\frac{1}{10000\,CL^2}
\left(
\frac{3}{m_p}
\right)^2.
\]
Then
\[
\begin{aligned}
\mathcal R_{1+(p-1)e^{T/(2C)}}(\mu_T\|\pi)
\le
\exp\left(
-\frac{T}{2Cm_p}
\right)
\mathcal R_p(\mu_0\|\pi)
+
200(p-1)NdL^2h^2.
\end{aligned}
\]
\end{lemma}

\begin{proof}
Define
\[
r(s)=1+(p-1)e^{s/(2C)},
\qquad
\Phi(s)
=
\frac{1}{r(s)}
\log\int\rho_s^{r(s)}\,d\pi,
\]
where $\rho_s=d\mu_s/d\pi$. The time restriction gives
\[
p\le r(s)\le m_p,
\qquad
0\le s\le T.
\]

For $s\in[kh,(k+1)h]$, Proposition~21 of
\citet{chewi2024analysislangevinmontecarlo}, together with
\eqref{eq:upper bound for D}, gives
\begin{align*}
\partial_s\Phi(s)
\le{}&
(r(s)-1)
\left(
-\frac{2}{r(s)^2}
+
36L^2(s-kh)^2
+
192hL^2C
\right)I(s)
\\
&+
18(r(s)-1)dL^3(s-kh)^2
+
84(r(s)-1)dL^2(s-kh),
\end{align*}
where
\[
I(s)
=
\frac{
\mathbb E_\pi[
\|\nabla(\rho_s^{r(s)/2})\|^2
]
}{
\mathbb E_\pi[\rho_s^{r(s)}]
}.
\]

The step-size condition implies
\[
36L^2h^2+192hL^2C
\le
\frac{1}{m_p^2}
\le
\frac{1}{r(s)^2}.
\]
Using the LSI lower bound
\[
I(s)
\ge
\frac{1}{2C}
\frac{r(s)}{r(s)-1}
\Phi(s)
\]
and $r(s)-1\le2(p-1)$, we obtain
\[
\partial_s\Phi(s)
\le
-\frac{1}{2Cm_p}\Phi(s)
+
36(p-1)dL^3(s-kh)^2
+
168(p-1)dL^2(s-kh).
\]

Integrating over each interval and iterating over the $N$ steps gives
\[
\Phi(T)
\le
\exp\left(
-\frac{T}{2Cm_p}
\right)
\Phi(0)
+
100(p-1)NdL^2h^2.
\]
Finally,
\[
\Phi(0)=\frac{p-1}{p}\mathcal R_p(\mu_0\|\pi),
\qquad
\mathcal R_{r(T)}(\mu_T\|\pi)
=
\frac{r(T)}{r(T)-1}\Phi(T).
\]
Since
\[
\frac{r(T)}{r(T)-1}\frac{p-1}{p}\le1,
\qquad
\frac{r(T)}{r(T)-1}\le2,
\]
the stated bound follows.
\end{proof}

We next extend the moving-target recursion.

\begin{proposition}[General-order moving-target recursion]
\label{prop:general-order-moving-targets}
Consider LMC with moving targets in
\eqref{ULA with moving targets}. Fix $p\ge2$ and write
\[
m_p=2p-1,
\qquad
T_i=N_i h_i.
\]
Assume that $\pi_{\lambda_i}$ satisfies an LSI with constant
\[
C_i=C_{\mathrm{LSI}}(\pi_{\lambda_i})\ge1
\]
and that $\nabla V_{\lambda_i}$ is $L_i$-Lipschitz with $L_i\ge1$.
Suppose
\[
T_i\le2C_i\log2,
\qquad
h_i
\le
\frac{1}{10000\,C_iL_i^2}
\left(
\frac{3}{m_p}
\right)^2.
\]
Then
\[
\begin{aligned}
\mathcal R_p(\mu_{t_i}\|\pi_{\lambda_{i+1}})
\le{}&
\exp\left(
-\frac{T_i}{4C_i m_p}
\right)
\mathcal R_p(\mu_{t_{i-1}}\|\pi_{\lambda_i})
\\
&+
400(p-1)N_i dL_i^2h_i^2
+
\mathcal R_{8C_i m_p/T_i}
(\pi_{\lambda_i}\|\pi_{\lambda_{i+1}}).
\end{aligned}
\]
\end{proposition}

\begin{proof}
Define
\[
r_i(s)
=
1+(p-1)e^{s/(2C_i)}
\]
and choose the interpolation parameter in the weak triangle
inequality as
\[
\vartheta_i
=
1-
\frac{(p-1)T_i}{4C_i m_p}.
\]
The time restriction implies $\vartheta_i\in(0,1)$.

The weak triangle inequality gives
\[
\begin{aligned}
\mathcal R_p(\mu_{t_i}\|\pi_{\lambda_{i+1}})
\le{}&
\frac{p-\vartheta_i}{p-1}
\mathcal R_{p/\vartheta_i}
(\mu_{t_i}\|\pi_{\lambda_i})
\\
&+
\mathcal R_{(p-\vartheta_i)/(1-\vartheta_i)}
(\pi_{\lambda_i}\|\pi_{\lambda_{i+1}}).
\end{aligned}
\]

We first compare the order in the first term with $r_i(T_i)$. Let
\[
x_i=\frac{T_i}{2C_i}.
\]
Since $r_i(T_i)\le m_p$ and $e^{x_i}-1\ge x_i$,
\[
r_i(T_i)-p
=
(p-1)(e^{x_i}-1)
\ge
(1-\vartheta_i)r_i(T_i).
\]
Hence
\[
\frac{p}{\vartheta_i}
\le
r_i(T_i).
\]
By monotonicity and
Lemma~\ref{lem:general-order-ula},
\[
\begin{aligned}
\mathcal R_{p/\vartheta_i}
(\mu_{t_i}\|\pi_{\lambda_i})
\le{}&
\exp\left(
-\frac{T_i}{2C_i m_p}
\right)
\mathcal R_p(\mu_{t_{i-1}}\|\pi_{\lambda_i})
\\
&+
200(p-1)N_i dL_i^2h_i^2.
\end{aligned}
\]

Moreover,
\[
\frac{p-\vartheta_i}{p-1}
=
1+\frac{T_i}{4C_i m_p},
\]
and
\[
\frac{p-\vartheta_i}{1-\vartheta_i}
=
1+\frac{4C_i m_p}{T_i}
\le
\frac{8C_i m_p}{T_i}.
\]
Writing
\[
u_i=\frac{T_i}{4C_i m_p},
\]
we also have
\[
(1+u_i)e^{-2u_i}\le e^{-u_i},
\qquad
1+u_i\le2.
\]
Substitution into the weak triangle inequality proves the result.
\end{proof}

\subsection{Proof of \autoref{thm:Renyi divergence calculation}}\label{appendix:Renyi divergence calculation}

\begin{proof}[Proof of \autoref{thm:Renyi divergence calculation}]
The R\'enyi divergence can be expressed as
\begin{align*}
\mathcal{R}_\alpha(\pi_\lambda\|\pi_\theta)
&=\frac{1}{\alpha-1}\log\int \left(\frac{d\pi_\lambda}{d\pi_\theta}\right)^\alpha d\pi_\theta\\
&=\frac{1}{\alpha-1}(-\alpha\log \frac{Z_\lambda}{Z_\theta} +\log\int e^{-\alpha\delta V} d\pi_\theta)\\
&= \frac{1}{\alpha-1}\Bigl[-\alpha\log\mathbb{E}_{\pi_\theta}[e^{-\delta V}] + \log\mathbb{E}_{\pi_\theta}[e^{-\alpha\delta V}]\Bigr].
\end{align*}
Define $\phi(z)=\log\mathbb{E}_{\pi_\theta}[e^{-z\delta V}]$. Then
\begin{align}
\mathcal{R}_\alpha(\pi_\lambda\|\pi_\theta) = \frac{1}{\alpha-1}\bigl[\phi(\alpha)-\alpha\phi(1)\bigr]. \label{eq:Renyi 1}
\end{align}
Expand $\phi$ around $0$ using Taylor's theorem:
\[
\phi(\alpha)=\phi(0)+\alpha\phi'(0)+\frac{\alpha^2}{2}\phi''(\xi_1),\quad \xi_1\in(0,\alpha),
\]
\[
\phi(1)=\phi(0)+\phi'(0)+\frac12\phi''(\xi_2),\quad \xi_2\in(0,1).
\]
Substituting into \eqref{eq:Renyi 1} gives
\begin{align}
\mathcal{R}_\alpha(\pi_\lambda\|\pi_\theta) = \frac{\alpha}{2(\alpha-1)}\Bigl[\alpha\phi''(\xi_1)-\phi''(\xi_2)\Bigr]. \label{eq:Renyi 2}
\end{align}
For any $z\in[0,\alpha]$, we have
\[
\phi''(z)=\operatorname{Var}_{\nu_z}(\delta V)\le \mathbb{E}_{\nu_z}[(\delta V)^2],\qquad 
d\nu_z = \frac{e^{-z\delta V}}{\mathbb{E}_{\pi_\theta}[e^{-z\delta V}]}d\pi_\theta.
\]
Using $|\delta V|\le\eta\phi$,
\[
\phi''(z) \le \eta^2\,\frac{\mathbb{E}_{\pi_\theta}[\phi^2 e^{-z\delta V}]}{\mathbb{E}_{\pi_\theta}[e^{-z\delta V}]}.
\]
For the numerator, $e^{-z\delta V}\le e^{z|\delta V|}\le e^{z\eta\phi}$. 
For the denominator, by Jensen's inequality,
\[
\mathbb{E}_{\pi_\theta}[e^{-z\delta V}]\ge e^{-z\,\mathbb{E}_{\pi_\theta}[\delta V]} \ge e^{-z\eta\,\mathbb{E}_{\pi_\theta}[\phi]},
\]
where we used $\delta V\le \eta\phi$ and $\mathbb{E}_{\pi_\theta}[\delta V]\le \eta\,\mathbb{E}_{\pi_\theta}[\phi]$.
Thus
\[
\phi''(z) \le \eta^2\, e^{z\eta\,\mathbb{E}_{\pi_\theta}[\phi]}\,\mathbb{E}_{\pi_\theta}[\phi^2 e^{z\eta\phi}].
\]
By the step-size condition $z\eta\le\alpha\eta\le\tau$, we obtain
\begin{align}
\phi''(z) \le \eta^2\, e^{\tau\,\mathbb{E}_{\pi_\theta}[\phi]}\,\mathbb{E}_{\pi_\theta}[\phi^2 e^{\tau\phi}] =: C_{\text{var}}\,\eta^2. \label{eq:Renyi 3}
\end{align}
Insert \eqref{eq:Renyi 3} into \eqref{eq:Renyi 2}:
\[
\mathcal{R}_\alpha(\pi_\lambda\|\pi_\theta) \le \frac{\alpha}{2(\alpha-1)}\bigl(\alpha C_{\text{var}}\eta^2 + C_{\text{var}}\eta^2\bigr)
= \frac{\alpha(\alpha+1)}{2(\alpha-1)}\,C_{\text{var}}\,\eta^2.
\]
For $\alpha\ge2$, note that $\frac{\alpha(\alpha+1)}{2(\alpha-1)}\le 2\alpha$, which yields the simplified bound.
\end{proof}

\subsection{Proof of \autoref{thm:Convergence of ULA using successive Moreau Envelopes}}\label{appendix:Convergence of ULA using successive Moreau Envelopes}
\begin{proof}[Proof of \autoref{thm:Convergence of ULA using successive Moreau Envelopes}]

Applying \citet[Theorem~4]{chewi2024analysislangevinmontecarlo}
to the fixed target $\pi_{\lambda_1}$ with R\'enyi order $q=3$,
the stated choices of $h$ and $\hat N$ give
\[
\mathcal R_2(\mu_0\|\pi_{\lambda_1})
\le
\mathcal R_3(\mu_0\|\pi_{\lambda_1})
\le
\frac{\epsilon}{6}.
\]
Here, the logarithmic factor in $\hat N$ follows by requiring the
exponentially decaying initialization term in the error bound of that
theorem to be of order $\epsilon$.

Since $n_0\ge 240000^2C_{\mathrm{LSI}}^4\gamma^{-4}$,
    \[
    h_n \le \frac{1}{10000\,C_{\mathrm{LSI}}(L_{f}+\frac{1}{\lambda_n})^2}.
    \]
By \autoref{thm:Renyi convergence for ULA with moving targets},
\begin{align}
\mathcal{R}_{2}(\mu_{n}\|\pi_{\lambda_{n+1}})
\le e^{-h_n/(12C_{\mathrm{LSI}})}\mathcal{R}_{2}(\mu_{n-1}\|\pi_{\lambda_{n}}) + 400d(L_f+\frac{1}{\lambda_n})^2h_n^2+\mathcal{R}_{\frac{24C_{\mathrm{LSI}}}{h_n}}(\pi_{\lambda_{n}}\|\pi_{\lambda_{n+1}}).\label{eq:recursion}
\end{align}

Now we need to bound the R\'enyi divergence between $\pi_{\lambda_{n}}$ and $\pi_{\lambda_{n+1}}$. By the definition of $\lambda_n$, 
\begin{align*}
    \lambda_n-\lambda_{n+1}&=\gamma((n+1+n_0)^{\frac{1}{4}}-(n+n_0)^{\frac{1}{4}})((n+n_0)^{\frac{1}{4}}-\gamma L_f)^{-1}((n+1+n_0)^{\frac{1}{4}}-\gamma L_f)^{-1}\\
    &\le \gamma\frac{1}{4}(n+n_0)^{-\frac{3}{4}}((n+n_0)^{\frac{1}{4}}-\frac{1}{2}(n+n_0)^{\frac{1}{4}})^{-2}\\
    &=\gamma(n+n_0)^{-\frac{5}{4}}.
\end{align*}
Thus,
\begin{align*}
    \frac{24C_{\mathrm{LSI}}}{h_n}|\lambda_n-\lambda_{n+1}|\le \gamma(n+n_0)^{-\frac{1}{4}}\le\tau.
\end{align*}

Then the assumption on the R\'enyi divergence gives:
\begin{align*}
    \mathcal{R}_{\frac{24C_{\mathrm{LSI}}}{h_n}}(\pi_{\lambda_{n}}\|\pi_{\lambda_{n+1}})\le \frac{24C_{\mathrm{LSI}}}{h_n}C_{\text{R\'enyi}}\,\gamma^2(n+n_0)^{-\frac{5}{2}}=C_{\text{R\'enyi}}\,\gamma^2(n+n_0)^{-\frac{3}{2}}.
\end{align*}
Thus \eqref{eq:recursion} gives
\begin{align*}
\mathcal{R}_{2}(\mu_{n}\|\pi_{\lambda_{n+1}})
&\le e^{-\frac{2}{n+n_0}}\mathcal{R}_{2}(\mu_{n-1}\|\pi_{\lambda_{n}}) + 400d\gamma^{-2}(n+n_0)^{\frac{1}{2}}(\frac{24C_{\mathrm{LSI}}}{n+n_0})^2+C_{\text{R\'enyi}}\,\gamma^2(n+n_0)^{-\frac{3}{2}}\\
&\le (1-\frac{1}{n+n_0})\mathcal{R}_{2}(\mu_{n-1}\|\pi_{\lambda_{n}}) + (250000dC_{\mathrm{LSI}}^2\gamma^{-2}+C_{\text{R\'enyi}}\,\gamma^2)(n+n_0)^{-\frac{3}{2}}.
\end{align*}
Multiplying both sides by $n+n_0$, we obtain
\begin{align*}
(n+n_0)\mathcal{R}_{2}(\mu_{n}\|\pi_{\lambda_{n+1}})
\le (n+n_0-1)\mathcal{R}_{2}(\mu_{n-1}\|\pi_{\lambda_{n}}) + (250000dC_{\mathrm{LSI}}^2\gamma^{-2}+C_{\text{R\'enyi}}\,\gamma^2)(n+n_0)^{-\frac{1}{2}}.
\end{align*}
Iterating the above recurrence, we obtain
\begin{align*}
(N+n_0)\mathcal{R}_{2}(\mu_{N}\|\pi_{\lambda_{N+1}})
\le n_0\mathcal{R}_{2}(\mu_{0}\|\pi_{\lambda_{1}}) + 2(250000dC_{\mathrm{LSI}}^2\gamma^{-2}+C_{\text{R\'enyi}}\,\gamma^2)(N+n_0)^{\frac{1}{2}}.
\end{align*}
Thus
\begin{align*}
\mathcal{R}_{2}(\mu_{N}\|\pi_{\lambda_{N+1}})
&\le \frac{n_0\mathcal{R}_{2}(\mu_{0}\|\pi_{\lambda_{1}})}{N+n_0} + 2(250000dC_{\mathrm{LSI}}^2\gamma^{-2}+C_{\text{R\'enyi}}\,\gamma^2)(N+n_0)^{-\frac{1}{2}}\\
&\le \frac{\epsilon}{6} + 2(250000dC_{\mathrm{LSI}}^2\gamma^{-2}+C_{\text{R\'enyi}}\,\gamma^2)(N+n_0)^{-\frac{1}{2}}\\
&\le \frac{\epsilon}{3}.
\end{align*}
In the first inequality, we use $\mathcal{R}_{2}(\mu_0\|\pi_{\lambda_1})\le \frac{\epsilon}{6}$. The second inequality follows from the choice of $N$.

From the definition of $n_0$, we have $2\lambda_{N+1}\le \tau$. Thus by the assumption on the R\'enyi divergence, 
\begin{align*}
    \mathcal{R}_2(\pi_{\lambda_{N+1}}\|\pi) \le 2\,C_{\text{R\'enyi}}\,(\lambda_{N+1}-0)^2\le 8\,C_{\text{R\'enyi}}\,\gamma^2(N+n_0)^{-\frac{1}{2}}\le \frac{\epsilon}{3}.
\end{align*}

By the weak triangle inequality in Lemma~\ref{lemma:renyi-basic},
\begin{align*}
    \mathcal{R}_{\frac{4}{3}}(\mu_{N}\|\pi)\le 2\mathcal{R}_{2}(\mu_{N}\|\pi_{\lambda_{N+1}})+\mathcal{R}_2(\pi_{\lambda_{N+1}}\|\pi) \le \epsilon.
\end{align*}

Finally, choosing
\[
\gamma^4
=
\frac{
dC_{\mathrm{LSI}}^2
}{
C_{\text{R\'enyi}}
}
\]
makes the two terms in the continuation complexity of the same order.
For this choice,
\[
C_{\mathrm{LSI}}\gamma^{-2}
=
\sqrt{
\frac{
C_{\text{R\'enyi}}
}{
d
}
}.
\]
Substituting these identities into the stated bounds for $\hat N$ and
$N$ gives the final total-complexity claim.

\end{proof}

\end{document}